\documentclass{article}

\usepackage[round,authoryear]{natbib}
\usepackage[preprint,nonatbib]{neurips_2019}

\usepackage[utf8]{inputenc} % allow utf-8 input
\usepackage[T1]{fontenc}    % use 8-bit T1 fonts
\usepackage{hyperref}       % hyperlinks
\usepackage{url}            % simple URL typesetting
\usepackage{booktabs}       % professional-quality tables
\usepackage{amsfonts}       % blackboard math symbols
\usepackage{nicefrac}       % compact symbols for 1/2, etc.
\usepackage{microtype}      % microtypography

\usepackage{graphicx}
\usepackage{wrapfig}
\usepackage{subcaption}
\usepackage{amsmath}
\usepackage{amsthm}
\usepackage{tikz}
\usepackage{filecontents}
\usepackage{algorithm}
\usepackage{algorithmic}
\usepackage{float}

\newcommand{\beq}{\begin{equation}}
\newcommand{\eeq}{\end{equation}}
\newcommand{\be}{\begin{equation}}
\newcommand{\ee}{\end{equation}}
\newcommand{\beqa}{\begin{eqnarray}}
\newcommand{\eeqa}{\end{eqnarray}}
\newcommand{\bean}{\begin{eqnarray*}}
\newcommand{\eean}{\end{eqnarray*}}

\newtheorem{proposition}{Proposition}

\theoremstyle{definition}
\newtheorem{example}{Example}

\newcommand{\cA}{{\mathcal A}}

\newcommand{\cJ}{{\mathcal J}}

\newcommand{\cL}{{\mathcal L}}

\newcommand{\cS}{{\mathcal S}}
\newcommand{\cT}{{\mathcal T}}

\newcommand{\vh}{{\mathbf h}}
\newcommand{\vp}{{\mathbf p}}
\newcommand{\vx}{{\mathbf x}}
\newcommand{\vX}{{\mathbf X}}
\newcommand{\vW}{{\mathbf W}}

\newcommand{\E}{{\mathbb E}}
\newcommand{\Proba}{{\mathbb P}}
\newcommand{\bR}{{\mathbb R}}
\def\to{\rightarrow}

\hypersetup{draft}

\title{Learning the Arrow of Time}

\author{
Nasim Rahaman$^{1, 2, 3}$ \; Steffen Wolf$^{1}$
\;  Anirudh Goyal$^{2}$ \; Roman Remme$^1$ \; 
Yoshua Bengio$^2$ \\ \\ 
\vspace{-0.3cm} 
\\ %
$^1$ Image Analysis and Learning Lab \\
\hspace{0.0cm} Ruprecht-Karls-Universit\"at \\
\hspace{0.0cm} Heidelberg, Germany \\
\vspace{0.0cm}\\
\hspace{0.0cm} $^2$ Mila \\
\hspace{0.0cm} Montr\'eal, Quebec, Canada 
\\
\vspace{0.0cm}\\
\hspace{0.0cm} $^3$ Max-Planck Institute for Intelligent Systems \\
\hspace{0.0cm} T\"ubingen, Germany 
}

\begin{document}

\maketitle
\vspace{-5pt}
\begin{abstract}
\vspace{-5pt}
We humans seem to have an innate understanding of the asymmetric progression of time, which we use to efficiently and safely perceive and manipulate our environment. Drawing inspiration from that, we address the problem of learning an \emph{arrow of time} in a Markov (Decision) Process. We illustrate how a learned arrow of time can capture meaningful information about the environment, which in turn can be used to measure reachability, detect side-effects and to obtain an intrinsic reward signal. We show empirical results on a selection of discrete and continuous environments, and demonstrate for a class of stochastic processes that the learned arrow of time agrees reasonably well with a known notion of an arrow of time given by the celebrated Jordan-Kinderlehrer-Otto result. 
\end{abstract}

\vspace{-10pt}
\section{Introduction} \label{intro}
The asymmetric progression of time has a profound effect on how we, as agents, perceive, process and manipulate our environment. Given a sequence of observations of our familiar surroundings (e.g. as photographs or video frames), we possess the innate ability to predict whether the said observations are ordered \emph{correctly}. We use this ability not just to perceive, but also to act: for instance, we know to be cautious about dropping a vase, guided by the intuition that the act of breaking a vase cannot be undone. It is manifest that this profound intuition reflects some fundamental properties of the world in which we dwell, and in this work, we ask whether and how these properties can be exploited to learn a representation that functionally mimics our understanding of the asymmetric nature of time. 

In his book \emph{The Nature of Physical World} \cite{eddington1929nature}, British astronomer Sir Arthur Stanley Eddington coined the term \emph{Arrow of Time} to denote this inherent asymmetry. It was attributed to the non-decreasing nature of the total thermodynamic entropy of an isolated system, as required by the second law of thermodynamics. However, the mathematical groundwork required for its description was already laid by \citet{lyapunov1892general} in the context of dynamical systems. Since then, the notion of an arrow of time has been formalized and explored in various contexts, spanning not only physics \citep{prigogine1978time, jordan1998variational, crooks1999entropy} but also algorithmic information theory \citep{zurek1989algorithmic, zurek1998decoherence}, causal inference \citep{janzing2016algorithmic} and time-series analysis \citep{janzing2010entropy, bauer2016arrow} . 

Expectedly, the notion of irreversiblity plays a central role in the discourse. In his Nobel lecture, \citet{prigogine1978time} posits that irreversible processes induce the arrow of time\footnote{Even for systems that are reversible at the microscopic scale, the unified integral fluctuation theorem \citep{seifert2012stochastic} shows that the ratio of the probability of a trajectory and its time-reversed counterpart grows exponentially with the amount of entropy the former produces.}. At the same time, the matter of reversibility has received considerable attention in reinforcement learning, especially in the context of safe exploration \citep{hans2008safe, moldovan2012safe, eysenbach2017leave}, learning backtracking models \citep{goyal2018recall, nair2018time} and AI-Safety \citep{amodei2016concrete, krakovna2018measuring}. In these applications, \emph{learning} a notion of (ir)reversibility is of paramount importance: for instance, the central premise of safe exploration is to avoid states that prematurely and irreversibly terminate the agent and/or damage the environment. It is related (but not identical) to the problem of detecting and avoiding side-effects, in particular those that adversely affect the environment. In \citet{amodei2016concrete}, the example considered is that of a cleaning robot tasked with moving a box across a room. The optimal way of successfully completing the task might involve the robot doing something disruptive, like knocking a vase over. Such disruptions might be difficult to recover from; in the extreme case, they might be virtually irreversible -- say when the vase is broken. 

The scope of this work includes detecting and quantifying such disruptions by \emph{learning the arrow of time} of an environment\footnote{Detecting the arrow of time in videos has been studied \citep{wei2018learning, pickup2014seeing}.}. Concretely, we aim to learn a \emph{potential} (scalar) function on the state space. This function must \emph{keep track of the passage of time}, in the sense that states that tend to occur in the \emph{future} -- states with a larger number of broken vases, for instance -- should be assigned larger values. To that end, we first introduce a general objective functional (Section~\ref{sec:formalism}) and study it analytically for toy problems (Section~\ref{sec:theoretical_analysis}). We continue by interpreting the solution to the objective (Section~\ref{sec:interpretation}) and highlight its applications (Section~\ref{sec:applications}). To tackle more complex problems, we parameterize the potential function by a neural network and present a stochastic training algorithm (Section~\ref{sec:experiments}). Subsequently, we demonstrate results on a selection of discrete and continuous environments and discuss the results critically, highlighting both the strengths and shortcomings of our method (Section~\ref{sec:experiments}). Finally, we place our method in a broader context by empirically elucidating connections to the theory of stochastic processes and the variational Fokker-Planck equation (Section \ref{sec:ito_sdes}). 

\section{The $h$-Potential}
\subsection{Formalism} \label{sec:formalism}
\textbf{Preliminaries.} In this section, we will represent the arrow of time as a scalar function $h$ that increases (in expectation) over time. Given a Markov Decision Process (\emph{Environment}), let $\cS$ and $\cA$ be its state and action space (respectively). A policy $\pi$ is a function mapping a state $s \in \cS$ to a distribution over the action space, $\pi(a | s) \in \Proba(\cA)$. Given $\pi$ and some distribution over the states, we call the sequence $(s_0, s_1, ..., s_{N})$ a state-transition trajectory $\tau_{\pi}$, where we have $s_{t + 1} \sim p_{\pi}(s_{t + 1}|s_{t}) = \sum_{a} p(s_{t + 1}|s_{t}, a_t)\pi(a_t | s_t)$ and $s_0 \sim p^{0}(s)$ for some initial state distribution $p^{0}(s)$. In this sense, $\tau_{\pi}$ can be thought of as an instantiation of the Markov (stochastic) process with transitions characterized by $p_{\pi}$. 

\textbf{Methods.} Now, for any given function $h: \mathcal{S} \to \mathbb{R}$, one may define the following functional: 
\begin{align} \label{eq:objective_wo_reg}
\cJ_{\pi}[h] &= \E_{t}\E_{(s_{t} \to s_{t + 1}) \sim \tau^{\pi}}[h(s_{t + 1}) - h(s_t)] = \E_{t} \E_{s_t} \E_{s_{t + 1}}[h(s_{t + 1}) - h(s_t) | s_t]
\end{align}
where the expectation is over the state transitions of the Markov process $\tau_{\pi}$ and the time-step $t$. We now define: 
\beq \label{eq:objective_w_reg}
h_{\pi} = \arg \max \{\cJ_{\pi}[h] + \lambda \cT[h] \}
\eeq
where $\cT$ implements some regularizer (e.g. $L_2$, etc.) weighted by $\lambda$. The $\cJ_{\pi}$ term is maximized if the quantity $h_{\pi}(s_t)$ increases in expectation\footnote{Note that while $\cJ_{\pi}[h]$ requires $h$ to increase along all trajectories \emph{in expectation}, it does not guarantee that it must increase along \emph{all} trajectories.} with increasing $t$, whereas the regularizer $\cT$ ensures that $h_{\pi}$ is well-behaved and does not diverge to infinity in a finite domain\footnote{Under certain conditions, $h_{\pi}$ resembles the negative stochastic discrete-time Lyapunov function \citep{li2013stability} of the Markov process $\tau_{\pi}$.}. In what follows, we simplify notation by using $h_{\pi}$ and $h$ interchangeably. 
\subsection{Theoretical Analysis} \label{sec:theoretical_analysis}
\begin{wrapfigure}{r}{0.35\textwidth}
\vspace{-10pt}
\centering
\resizebox{0.35\textwidth}{!}{\tikzset{every picture/.style={line width=0.75pt}} %set default line width to 0.75pt        

\begin{tikzpicture}[x=0.75pt,y=0.75pt,yscale=-1,xscale=1]
%uncomment if require: \path (0,300); %set diagram left start at 0, and has height of 300

%Shape: Ellipse [id:dp9905966811127849] 
\draw   (120,149.28) .. controls (120,136.18) and (130.62,125.56) .. (143.72,125.56) .. controls (156.82,125.56) and (167.44,136.18) .. (167.44,149.28) .. controls (167.44,162.38) and (156.82,173) .. (143.72,173) .. controls (130.62,173) and (120,162.38) .. (120,149.28) -- cycle ;
%Shape: Ellipse [id:dp5623625050780503] 
\draw   (263.81,149.28) .. controls (263.81,136.18) and (274.43,125.56) .. (287.53,125.56) .. controls (300.63,125.56) and (311.25,136.18) .. (311.25,149.28) .. controls (311.25,162.38) and (300.63,173) .. (287.53,173) .. controls (274.43,173) and (263.81,162.38) .. (263.81,149.28) -- cycle ;
%Curve Lines [id:da21015443611954532] 
\draw    (159.83,131.33) .. controls (190.19,100.31) and (239.5,100) .. (272.63,129.11) ;
\draw [shift={(273.63,130.01)}, rotate = 222.17000000000002] [fill={rgb, 255:red, 0; green, 0; blue, 0 }  ][line width=0.75]  [draw opacity=0] (8.93,-4.29) -- (0,0) -- (8.93,4.29) -- cycle    ;

%Curve Lines [id:da16547736645920996] 
\draw    (161.15,170.47) .. controls (191.24,199.35) and (241.86,199.39) .. (275.83,169.33) ;

\draw [shift={(159.33,168.67)}, rotate = 45.79] [fill={rgb, 255:red, 0; green, 0; blue, 0 }  ][line width=0.75]  [draw opacity=0] (8.93,-4.29) -- (0,0) -- (8.93,4.29) -- cycle    ;
%Curve Lines [id:da9355443084734532] 
\draw    (302,130) .. controls (359.54,100.3) and (360.12,198.01) .. (301.91,170.38) ;
\draw [shift={(300.13,169.5)}, rotate = 386.95] [fill={rgb, 255:red, 0; green, 0; blue, 0 }  ][line width=0.75]  [draw opacity=0] (8.93,-4.29) -- (0,0) -- (8.93,4.29) -- cycle    ;

%Curve Lines [id:da29763408825433024] 
\draw    (128.63,131) .. controls (70.21,99.81) and (69.63,198) .. (126.87,169.89) ;
\draw [shift={(128.63,169)}, rotate = 512.28] [fill={rgb, 255:red, 0; green, 0; blue, 0 }  ][line width=0.75]  [draw opacity=0] (8.93,-4.29) -- (0,0) -- (8.93,4.29) -- cycle    ;

% Text Node
\draw (217,90) node [scale=1.2]  {$\alpha $};
% Text Node
\draw (320.67,190.67) node [scale=1.2]  {$\alpha $};
% Text Node
\draw (216.67,208) node [scale=1.2]  {$1-\alpha $};
% Text Node
\draw (106.67,104) node [scale=1.2]  {$1-\alpha $};
% Text Node
\draw (143.72,149.28) node [scale=1.2]  {$s_1$};
% Text Node
\draw (287.53,149.28) node [scale=1.2]  {$s_2$};

\end{tikzpicture}\unskip}
\caption{\small A two variable Markov chain where the reversibility of the transition from the first state to the second is parameterized by $\alpha$.}\label{fig:two_var_mc}
\vspace{-10pt}
\end{wrapfigure}
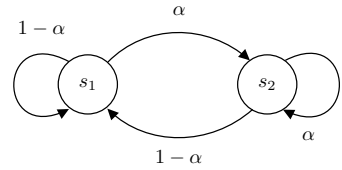

The optimization problem formulated in Eqn~\ref{eq:objective_w_reg} can be studied analytically: in Appendix~\ref{app:theoretical_analysis}, we derive the analytic solutions for Markov processes with discrete state-spaces and known transition matrices. The key result of our analysis is a characterization of how the optimal $h$ must behave for the considered regularization schemes. Further, we evaluate the solutions for illustrative toy Markov chains in Fig~\ref{fig:two_var_mc} and \ref{fig:four_var_mc} to learn the following. 

\textbf{First}, consider the two variable Markov chain in Fig~\ref{fig:two_var_mc} where the initial state is either $s_1$ or $s_2$ with equal probability. If $\alpha > 0.5$, the transition from $s_1$ to $s_2$ is more likely than the reverse transition from $s_2$ to $s_1$. In this case, one would expect that $h(s_1) < h(s_2)$ and that $h(s_2) - h(s_1)$ increases with $\alpha$, which is indeed what we find\footnote{cf. Examples~\ref{exp:two_var_mdp} and ~\ref{exp:two_var_mdp_tr} in Appendix~\ref{app:theoretical_analysis}.} given an appropriate regularizer. Conversely, if $\alpha = 0.5$, the transition between $s_1$ and $s_2$ is equally likely in either direction (i.e. it is \emph{fully reversible}), and we obtain $h(s_1) = h(s_2)$. \textbf{Second}, consider the four variable Markov chain in Fig~\ref{fig:four_var_mc} where the initial state is $s_1$ and all state transitions are irreversible. Inituitively, one should expect that $h(s_1) < h(s_2) < h(s_3) < h(s_4)$, and $h(s_2) - h(s_1) = h(s_3) - h(s_2) = h(s_4) - h(s_3)$, given that all state transitions are \emph{equally irreversible}. We obtain this behaviour with an appropriate regularizer\footnote{cf. Example~\ref{exp:four_var_mdp_tr} in Appendix~\ref{app:theoretical_analysis}.}. 

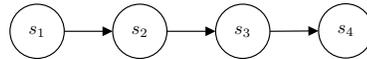
\begin{wrapfigure}{r}{0.35\textwidth}
\vspace{-10pt}
\centering
\resizebox{0.35\textwidth}{!}{\tikzset{every picture/.style={line width=0.75pt}} %set default line width to 0.75pt        

\begin{tikzpicture}[x=0.75pt,y=0.75pt,yscale=-1,xscale=1]
%uncomment if require: \path (0,300); %set diagram left start at 0, and has height of 300

%Shape: Ellipse [id:dp9905966811127849] 
\draw   (121,149.28) .. controls (121,136.18) and (131.62,125.56) .. (144.72,125.56) .. controls (157.82,125.56) and (168.44,136.18) .. (168.44,149.28) .. controls (168.44,162.38) and (157.82,173) .. (144.72,173) .. controls (131.62,173) and (121,162.38) .. (121,149.28) -- cycle ;
%Shape: Ellipse [id:dp5623625050780503] 
\draw   (210.81,149.28) .. controls (210.81,136.18) and (221.43,125.56) .. (234.53,125.56) .. controls (247.63,125.56) and (258.25,136.18) .. (258.25,149.28) .. controls (258.25,162.38) and (247.63,173) .. (234.53,173) .. controls (221.43,173) and (210.81,162.38) .. (210.81,149.28) -- cycle ;
%Shape: Ellipse [id:dp45033858416678174] 
\draw   (300.81,149.28) .. controls (300.81,136.18) and (311.43,125.56) .. (324.53,125.56) .. controls (337.63,125.56) and (348.25,136.18) .. (348.25,149.28) .. controls (348.25,162.38) and (337.63,173) .. (324.53,173) .. controls (311.43,173) and (300.81,162.38) .. (300.81,149.28) -- cycle ;
%Shape: Ellipse [id:dp7899890428091689] 
\draw   (390.81,148.95) .. controls (390.81,135.85) and (401.43,125.22) .. (414.53,125.22) .. controls (427.63,125.22) and (438.25,135.85) .. (438.25,148.95) .. controls (438.25,162.05) and (427.63,172.67) .. (414.53,172.67) .. controls (401.43,172.67) and (390.81,162.05) .. (390.81,148.95) -- cycle ;
%Straight Lines [id:da5003551293395017] 
\draw    (168.44,149.28) -- (208.81,149.28) ;
\draw [shift={(210.81,149.28)}, rotate = 180] [fill={rgb, 255:red, 0; green, 0; blue, 0 }  ][line width=0.75]  [draw opacity=0] (8.93,-4.29) -- (0,0) -- (8.93,4.29) -- cycle    ;

%Straight Lines [id:da24248773440711413] 
\draw    (258.25,149.28) -- (298.81,149.28) ;
\draw [shift={(300.81,149.28)}, rotate = 180] [fill={rgb, 255:red, 0; green, 0; blue, 0 }  ][line width=0.75]  [draw opacity=0] (8.93,-4.29) -- (0,0) -- (8.93,4.29) -- cycle    ;

%Straight Lines [id:da6174513510642952] 
\draw    (348.25,149.28) -- (388.81,148.96) ;
\draw [shift={(390.81,148.95)}, rotate = 539.55] [fill={rgb, 255:red, 0; green, 0; blue, 0 }  ][line width=0.75]  [draw opacity=0] (8.93,-4.29) -- (0,0) -- (8.93,4.29) -- cycle    ;

% Text Node
\draw (144.72,149.28) node [scale=1.2]  {$s_1$};
% Text Node
\draw (234.53,149.28) node [scale=1.2]  {$s_2$};
% Text Node
\draw (324.53,149.28) node [scale=1.2]  {$s_3$};
% Text Node
\draw (414.53,148.95) node [scale=1.2]  {$s_4$};

\end{tikzpicture}\unskip}
\caption{\small A four variable Markov chain corresponding to a sequence of irreversible state transitions.}\label{fig:four_var_mc}
\vspace{-10pt}
\end{wrapfigure}
While this serves to show that the optimization problem defined in Eqn~\ref{eq:objective_w_reg} can indeed lead to interesting solutions, an analytical treatment is not always feasible for complex environments with a large number of states and/or undetermined state transition rules. In such cases, as we shall see in later sections, one may resort to parameterizing $h$ as a function approximator and solve the optimization problem in Eqn~\ref{eq:objective_w_reg} with stochastic gradient methods. 

\vspace{-5pt}
\subsection{Interpretation and Subtleties} \label{sec:interpretation}
\vspace{-5pt}
Having defined and analyzed $h$, we turn to the task of interpreting it. Based on the analytical results presented in Section~\ref{sec:theoretical_analysis}, it seems reasonable to expect that even in interesting environments, $h$ should remain constant (in expectation) along reversible trajectories. Further, along trajectories with irreversible transitions, one may hope that $h$ not only increases, but also \emph{quantifies} the irreversibility in some sense. In Section~\ref{sec:experiments}, we empirically investigate if this is indeed the case. But before that, there are two conceptual aspects that warrant closer scrutiny. 

The first is rooted in the observation that the states $s_t$ are collected by a given but arbitrary policy $\pi$. In particular, there may exist \emph{demonic}\footnote{This is indeed an allusion to Maxwell's Demon, cf. \citet{thomson1874kinetic}.} policies for which the resulting arrow-of-time is unnatural, perhaps even misleading. Consider for instance the actions of a practitioner of Kintsugi, the ancient Japanese art of repairing broken pottery. The corresponding policy might cause the environment to transition from a state where the vase is broken to one where it is not. If we learn $h$ on such \emph{demonic} (or expert) trajectories\footnote{By doing so, we solve an inverse RL problem \citep{ng2000algorithms}.}, it might be the case that counter to our intuition, states with a larger number of broken vases are assigned smaller values (and the vice versa). Now, we may choose to resolve this conundrum by defining 
\beq \label{eq:full_objective}
\cJ[h] = \E_{\pi \sim \Pi} \cJ_{\pi}[h]
\eeq where $\Pi$ is the set of all policies defined on $\cS$, and $\sim$ denotes uniform sampling. The resulting function $h^{*} = \arg \max \{\cJ[h] + \cT\}$ would characterize the arrow-of-time w.r.t. all possible policies, and one would expect that for a vast majority of such policies, the transition from broken vase to a intact vase is rather unlikely and/or requires highly specialized policies. 

Unfortunately, determining $h^{*}$ is not feasible for most interesting applications, given the outer expectation over \emph{all} possible policies; we therefore settle for a (uniformly) random policy which we denote by $\pi_{\sharp}$ (and the corresponding potential as $h_{\sharp}$). The simplicity (or rather, \emph{clumsiness}) of $\pi_{\sharp}$ justifies its adoption, since one would expect a \emph{demonic} policy to be rather complex and not implementable with random actions. In this sense, we ensure that the arrow of time characterizes the underlying dynamics of the environment, and not the peculiarities of a particular agent. However, the price we pay for our choice is the lack of adequate exploration in complex enough environments, although this problem plagues most model-based reinforcement learning approaches\footnote{While this is a fundamental problem, powerful methods for off-policy learning exist (cf. \citet{munos2016safe} and references therein); however, a full analysis is beyond the scope of the current work.} (cf. \citet{ha2018world}). 

The second aspect concerns what we require of environments in which the arrow of time is informative. To illustrate the matter, we consider a class of \emph{Hamiltonian} systems\footnote{Systems where Liouville's theorem holds. Further, the Hamiltonian is assumed to be time-independent.}, a typical instance of which could be a billiard ball moving on a frictionless arena and bouncing (elastically) off the edges\footnote{This is well studied in the context of dynamical systems and chaos theory (keyword: \emph{dynamic billiards}); see \citep{bunimovich2007billiards} and references therein.}. The state space comprises the ball's velocity and its position constrained to a \emph{billiard table} (without holes!), where the ball is initialized at a random position on the table. For such a system, it can be seen by time-reversal symmetry that when averaged over a large number of trajectories, the state transition $s \to s'$ is just as likely as the reverse transition $s' \to s$. In this case, one should expect the arrow of time to be constant\footnote{\citet{prigogine1978time} (page 783 et seq.) provides a more physical treatment.} (see Eqn~\ref{eq:objective_w_reg}). A similar argument can be made for systems that identically follow closed trajectories in their respective state space (e.g. a frictionless and undriven pendulum). It follows that $h$ must remain constant along the trajectory and that the arrow of time is uninformative. However, for so-called \emph{dissipative} systems, the notion of an arrow of time is pronounced and well studied \citep{prigogine1978time, willems1972dissipative}. In MDPs, dissipative behaviour may arise in situations where certain transitions are irreversible by design (e.g. bricks disappearing in Atari Breakout), or due to partial observability (e.g. for a damped pendulum, the state space does not track the microscopic processes that give rise to friction\footnote{Note that while a damped pendulum can be expressed as a Hamiltonian system \citep{mcdonald2015damped}, the Hamiltonian is time dependent.}). 

Therefore, a central premise underlying the practical utility of learning the arrow of time is that the considered MDP is indeed dissipative. Operating under this assumption, we now discuss a few applications of the arrow of time and experimentally demonstrate its learnability on non-trivial environments. 

\section{Applications with Related Work} \label{sec:applications}
\subsection{Measuring Reachability} \label{sec:reachability}
Given two states $s$ and $s'$ in $\cS$, the reachability of $s'$ from $s$ measures how difficult it is for an agent at state $s$ to reach state $s'$. The prospect of learning reachability from state-transition trajectories has been explored: in \citet{savinov2018episodic}, the approach taken involves learning a logistic regressor network $g^{\theta}: \cS \times \cS \to [0, 1]$ to predict the probability of states $s'$ and $s$ being reachable to one another within a certain number of steps (of a random policy), in which case $g(s, s') \approx 1$. However, the model $g$ is not \emph{directed}: it does not learn whether $s'$ is more likely to follow $s$, or the vice versa. Instead, we propose to learn a function $\eta_{\pi}: \cS \times \cS \to \bR$ such that $\eta_{\pi}(s \to s') \mapsto h_{\pi}(s') - h_{\pi}(s)$ where $\eta_{\pi}(s \to s')$ is said to measure the \emph{directed reachability} of state $s'$ from state $s$ by following some reference policy $\pi$. In the following, we take the reference policy as given (e.g. a random policy) and drop the $\pi$ for notational clarity. Now, $\eta$ has the following properties.

\textbf{First}, consider the case where the transition between states $s$ and $s'$ is fully reversible, i.e. when state $s$ is exactly as reachable from state $s'$ as is $s'$ from $s$. In expectation, we obtain $h(s') = h(s)$ and consequently, $\eta(s \to s') = \eta(s' \to s)$. We denote such reversible transitions with $s \leftrightarrow s'$. Now, if instead the state $s'$ is more likely to occur after state $s$ than state $s$ after $s'$, we say $s'$ is \emph{more reachable}\footnote{With respect to the reference (random) policy, which is implicit in our notation.} from $s$ than $s$ from $s'$. It follows in expectation that $h(s') > h(s)$, and consequently, $\eta(s \to s') > 0$ along with $\eta(s' \to s) = -\eta(s \to s') < 0$. \textbf{Second}, it can easily be seen that the reachability measure implemented by $\eta$ is additive by design: given three states $s_0, s_1, s_2 \in \cS$, we have that $\eta(s_0 \to s_2) = \eta(s_0 \to s_1) + \eta(s_1 \to s_2)$. As a special case, consider when $s_0 \leftrightarrow s_1$ and $s_1 \leftrightarrow s_2$: it follows that $s_0 \leftrightarrow s_2$. In words, if both transitions, from $s_0$ to $s_1$ and from $s_1$ and $s_2$, are fully reversible, it automatically follows that the transition from $s_0$ to $s_2$ is also fully reversible. \textbf{Third}, $\eta$ allows for a \emph{soft} measure of reachability. As we shall see in Section~\ref{sec:experiments}, it measures not only \emph{whether} a state $s'$ is reachable from another state $s$, but also quantifies \emph{how} reachable the former is from the latter. For instance: if the state $s_{(0)}$ is one with all vases intact, $s_{(1)}$ with one vase broken, and $s_{(100)}$ with a hundred vases broken, we find that $\eta(s_{(0)} \to s_{(100)}) \approx 100 \cdot \eta(s_{(0)} \to s_{(1)})$. This behaviour is sought-after in the context of AI-Safety \citep{krakovna2018measuring, leike2017ai}. 

While these properties are satisfactory, the following aspect should be considered to prevent potential confusion. Namely, while we expect $\eta(s' \to s) = \eta(s \to s')$ if the transition between states $s$ and $s'$ is fully reversible, the converse is not guaranteed, especially for non-ergodic environments. For instance, if a Markov chain does not admit a trajectory between states $s$ and $s'$, it might still be the case that $h(s) = h(s')$, and consequently, $\eta(s \to s') = \eta(s' \to s) = 0$. 

\subsection{Detecting and Penalizing Side Effects for Safe Exploration} \label{sec:safe_exp}
The problem of detecting and avoiding side-effects is well known and crucially important for safe exploration \citep{moldovan2012safe, eysenbach2017leave, krakovna2018measuring, armstrong2017low}. Broadly, the problem involves detecting and avoiding state transitions that permanently and irreversibly damage the agent or the environment. As such, it is not surprising that it is fundamentally related to reachability, as in the agent is prohibited from taking actions that drastically reduce the reachability between the resulting state and some predefined \emph{safe} state. In \citet{eysenbach2017leave}, the authors learn a reset policy responsible for resetting the environment to some initial state after the agent has completed its trajectory. The resulting value function of the reset policy indicates when the actual (\emph{forward}) policy executes an irreversible state transition. In contrast, \citet{krakovna2018measuring} propose to attack the problem by measuring reachability relative to a baseline state. However, determining it requires counterfactual reasoning, which in turn requires a known causal model. 

We propose to directly use the reachability measure $\eta$ defined in Section~\ref{sec:reachability} to derive a Lagrangian for safe-exploration. Let $r_t$ be the reward (potentially including an exploration bonus) at time-step $t$. The augmented reward is given by: 
\beq \label{eq:safe_reward}
\hat r_t = r_t - \beta \cdot \max\{\eta(s_{t - 1} \to s_{t}), 0\}
\eeq
where $\beta$ is a scaling coefficient. In practice, one may replace $\eta$ with $\sigma(\eta)$, where $\sigma$ is a monotonically increasing transfer function (e.g. a step function). 

Intuitively, transitions $s \to s'$ that are \emph{less reversible} cause the $h$-potential to increase, and the resulting reachability measure $\eta(s \to s') > 0$ in expectation. This in-turn incurs a penalty, which is reflected in the value function of the agent. Conversely, transitions that are reversible should have the property that $\eta(s \to s') = 0$ (also in expectation), thereby incurring no penalty. 
\subsection{Rewarding Curious Behaviour} \label{sec:curiousity}
In many environments where reinforcement learning methods shine, the reward function is assumed to be given; however, shaping a good reward function can often prove to be a challenging endeavour. It is in this context that the notion of \emph{curiosity} comes to play an important role \citep{schmidhuber2010formal, chentanez2005intrinsically, pathakICMl17curiosity, burda2018large, savinov2018episodic}. One typical approach towards encouraging curious behaviour is to seek \emph{novel} states that surprise the agent \citep{schmidhuber2010formal, pathakICMl17curiosity, burda2018large} and use the error in the agent's prediction of future states is used as a curiosity reward. This approach is, however, susceptible to the so-called noisy-TV problem, wherein an uninteresting source of entropy like a noisy-TV can induce a large curiosity bonus because the agent cannot predict its future state. \citet{savinov2018episodic} propose to define novelty in terms of (undirected) reachability - states that are easily reachable from the current state are considered less novel. 

The $h$-potential and the corresponding reachability measure $\eta$ affords another way of defining a curiosity reward: namely, states that are difficult to access by a simple reference policy (e.g. a random policy) should incur a larger reward. In other words, it encourages an agent to \emph{do hard things}, i.e. to seek states that are otherwise difficult to reach just by chance. The general form of the corresponding reward is given by: $\hat r_t = -\eta(s_{t - 1} \to s_t)$.

Despite the above being independent of the reward function defined by the environment, the latter might often align with the former: in many environments, the task at hand is to reach the least reachable state. This is readily recognized in classical control tasks like Pendulum, Cartpole and Mountain-Car, where the goal state is often the least reachable. However, if the environment's specified task requires the agent to inadvertently execute irreversible trajectories, we expect our proposed reward to be less applicable. 

\section{Algorithm and Experiments} \label{sec:experiments}
In this section, we introduce a learning algorithm for the parameterized $h$-potential and empirically validate it on a selection of discrete and continuous environments (more experiments can be found in Appendix~\ref{app:experiments}). 

For interesting MDPs with a large number of states and unknown state transition models, an analytic solution like in Section~\ref{sec:theoretical_analysis} is not feasible. In such cases, the $h$-potential can be parameterized by a neural network $h^{\theta}$ with parameters $\theta$, reducing the optimization problem in Eqn~\ref{eq:objective_w_reg} to: 
\beq \label{eq:practical_opt}
\arg \max_{\theta} \left\{\cJ_{\pi}[h^{\theta}] + \lambda \cT[h^{\theta}] \right\}
\eeq
For stochastic training, the expectation in Eqn~\ref{eq:objective_wo_reg} can be replaced by its Monte-Carlo estimate, and optimization problem in Eqn~\ref{eq:practical_opt} can be solved via stochastic gradient descent -- the details are given in Algorithm~\ref{alg:main} (Appendix ~\ref{app:algorithm}). 

Now, we turn to the question of what regularizer to use. Perhaps the simplest candidate is early stopping, wherein the network $h^{\theta}$ is simply not trained to convergence. In combination with weight-decay and/or gradient clipping, we find it to work surprisingly well in practice. Another good regularizer is the so-called \emph{trajectory regularizer} (cf. Eqn~\ref{eq:traj_reg} in Appendix~\ref{app:theoretical_analysis}): 
\beq \label{eq:traj_reg_cont}
\cT[h] = -\E_{t} \E_{s_t} \E_{s_{t + 1}}[(h(s_{t + 1}) - h(s_t))^2 | s_t]
\eeq
In words, the trajectory regularizer penalizes all changes in $h$, whereas the primary objective $\cJ$ encourages $h$ to increase along a trajectory; for an appropriate coefficient $\lambda$, a balance is found. 

\begin{wrapfigure}{r}{0.45\textwidth}
\vspace{-15pt}
\centering
\includegraphics[width=0.45\textwidth]{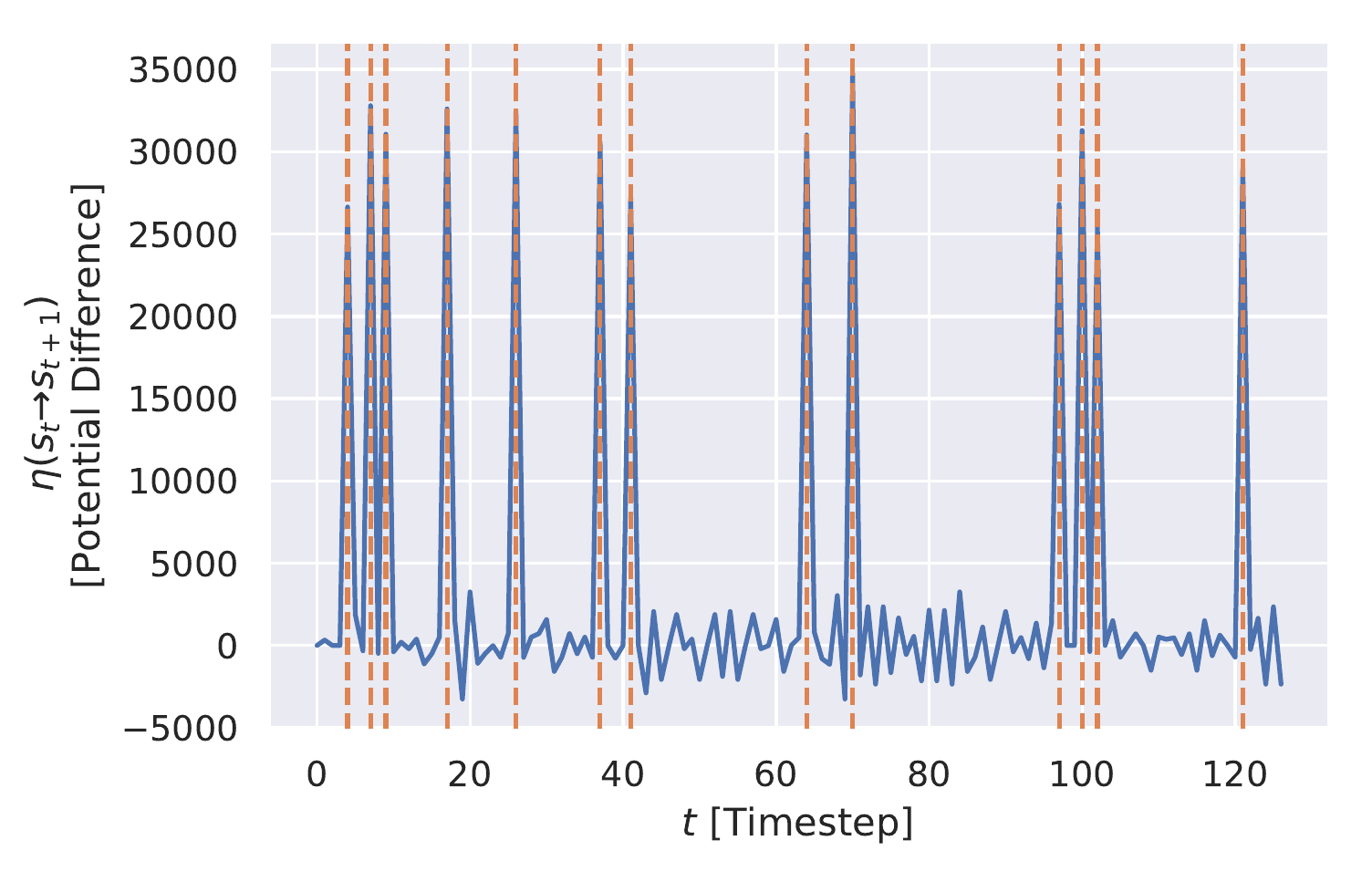}
\caption{\small The potential difference (i.e. change in $h$-potential) between consecutive states along a trajectory. The dashed vertical lines denote when a vase is broken. \textbf{Gist:} the $h$-potential increases step-wise when the agent irreversibly breaks a vase (corresponding to the spikes), but remains constant as it reversibly moves about. Further, the spikes are all of roughly the same height, indicating that the $h$-potential has learned to count the number of destroyed vases.}\label{fig:vaseworld_dpot_nonoise}
\vspace{-10pt}
\end{wrapfigure}

\textbf{2D World with Vases.\footnote{Experimental details and additional plots can be found in Appendix~\ref{app:vaseworld}.}} The environment considered is a $7 \times 7$ 2D world, where cells can be occupied by the agent, the goal and/or a vase (their respective positions are randomly sampled in each episode). If the agent enters a cell with a vase in it, the vase disappears without compromising the agent. We use a random policy to generate state-transition trajectories, which we then use to train the $h$-potential. In Fig~\ref{fig:vaseworld_pot_nonoise} (in Appendix \ref{app:vaseworld}), we plot the $h$-potential along a trajectory (parameterized by $t$) generated by a random policy. We find that $h(s_t)$ increases step-wise when the agent breaks a vase, but remains constant as it moves around -- consequently, we observe that the breaking of a vase corresponds to a spike in the $\eta(s_{t} \to s_{t + 1})$ signal (Fig~\ref{fig:vaseworld_dpot_nonoise}). Indeed, the latter is reversible whereas the former irreversibly changes the environment. Moreover, we find that the spikes in Fig~\ref{fig:vaseworld_dpot_nonoise} are of roughly similar heights, indicating that the model has learned to measure the number of vases broken, i.e. it has learned to \emph{quantify} irreversiblity, instead of merely detecting it\footnote{The trained $h$-potential can be utilized to derive a safety reward from the trained model, as elaborated in Section \ref{sec:safe_exp} (cf. Fig~\ref{fig:vaseworld_rl} in Appendix \ref{app:vaseworld}).}.

% % 
Now, to study the robustness of the $h$-potential to noise, we carry out the following two experiments. In the first of the two, we append a uniformly-random temporally uncorrelated noise signal to the state, which serves as an entropy source (i.e. a noisy-TV). In the second, we append a \emph{clock} to the state, i.e. a temporally-correlated signal that increases in constant intervals as the trajectory progresses. Fig~\ref{fig:vaseworld_dpot_tvnoise} and \ref{fig:vaseworld_dpot_causalnoise} (Appendix~\ref{app:vaseworld}) show the respective plots for the corresponding reachability $\eta$. While the former noises the background in the $\eta(s_{t} \to s_{t + 1})$ signal, the spikes remain clearly visible, suggesting that the $h$-potential is fairly robust to temporally uncorrelated sources of entropy. The latter has a more interesting effect - the $h$-potential latches on to the clock signal, which results in the baseline $\eta$ shifting up by a constant. While the spikes remain visible for the most part, this experiment shows that the model might be susceptible to spurious causal signals in the environment.

\begin{wrapfigure}{r}{0.45\textwidth}
\vspace{-10pt}
\centering
\includegraphics[width=0.45\textwidth]{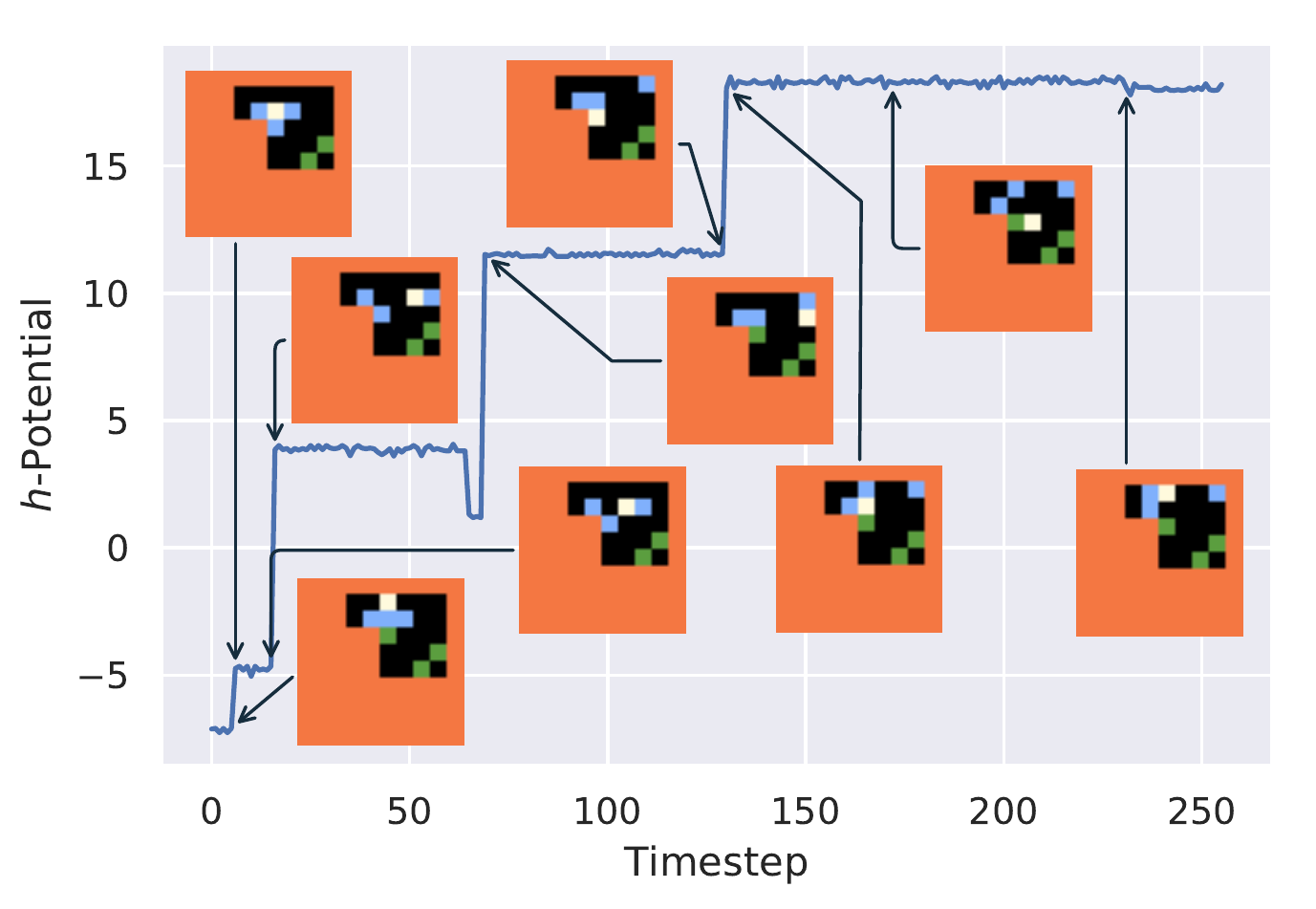}
\caption{\small The $h$-potential along a trajectory from a random policy, annotated with the corresponding state images. The white sprite corresponds to the agent, orange to a wall, blue to a box and green to a goal. \textbf{Gist:} the $h$-potential increases sharply as the agent pushes a box against the wall. While it may decrease (for a given trajectory) if the agent manages to move a box away from the wall (in this case), it increases in expectation over trajectories (cf. Fig~\ref{fig:sokoban_pot} in Appendix \ref{app:sokoban}).}\label{fig:sokoban_ent}
\vspace{-15pt}
\end{wrapfigure}
\textbf{Sokoban\footnote{Experimental details and additional plots can be found in Appendix~\ref{app:sokoban}.}} ("warehouse-keeper") is a classic puzzle video game, where an agent must push a number of boxes to set goal locations placed on a map. We use a 2D-world like implementation\footnote{Our implementation is adapted from \citet{schrader2018sokoban}.}, where each cell can be occupied by a wall, the agent or a box. Additionally, a goal marker may co-occupy a cell with all sprites except a wall. The agent may only push boxes (and not pull), rendering certain moves irreversible - for instance, when a box is pushed against a wall. Solving Sokoban requires long-term planning, precisely due to the existence of such irreversible moves. To further exacerbate the problem, the task of even determining whether a move is irreversible might be non-trivial. 

We train the $h$-potential on trajectories generated by a random policy, wherein we generate a random (solvable) map for each trajectory. Fig~\ref{fig:sokoban_ent} shows the evolution of $h$ with timesteps for a randomly sampled (validation) map. We find that $h$ increases sharply when a box is pushed against a wall, but remains constant as the agent moves about (potentially pushing a box around). Indeed, the latter is reversible whereas the former is not. Further, we confirm that $h$ does not necessarily increase along all trajectories, but only in expectation (Fig~\ref{fig:sokoban_pot}). We therefore learn that the $h$-potential can be used to extract useful information from the environment, all without any external supervision (via rewards) or specialized policies. 
\begin{wrapfigure}{r}{0.45\textwidth}
\vspace{-10pt}
\centering
  \includegraphics[width=0.45\textwidth]{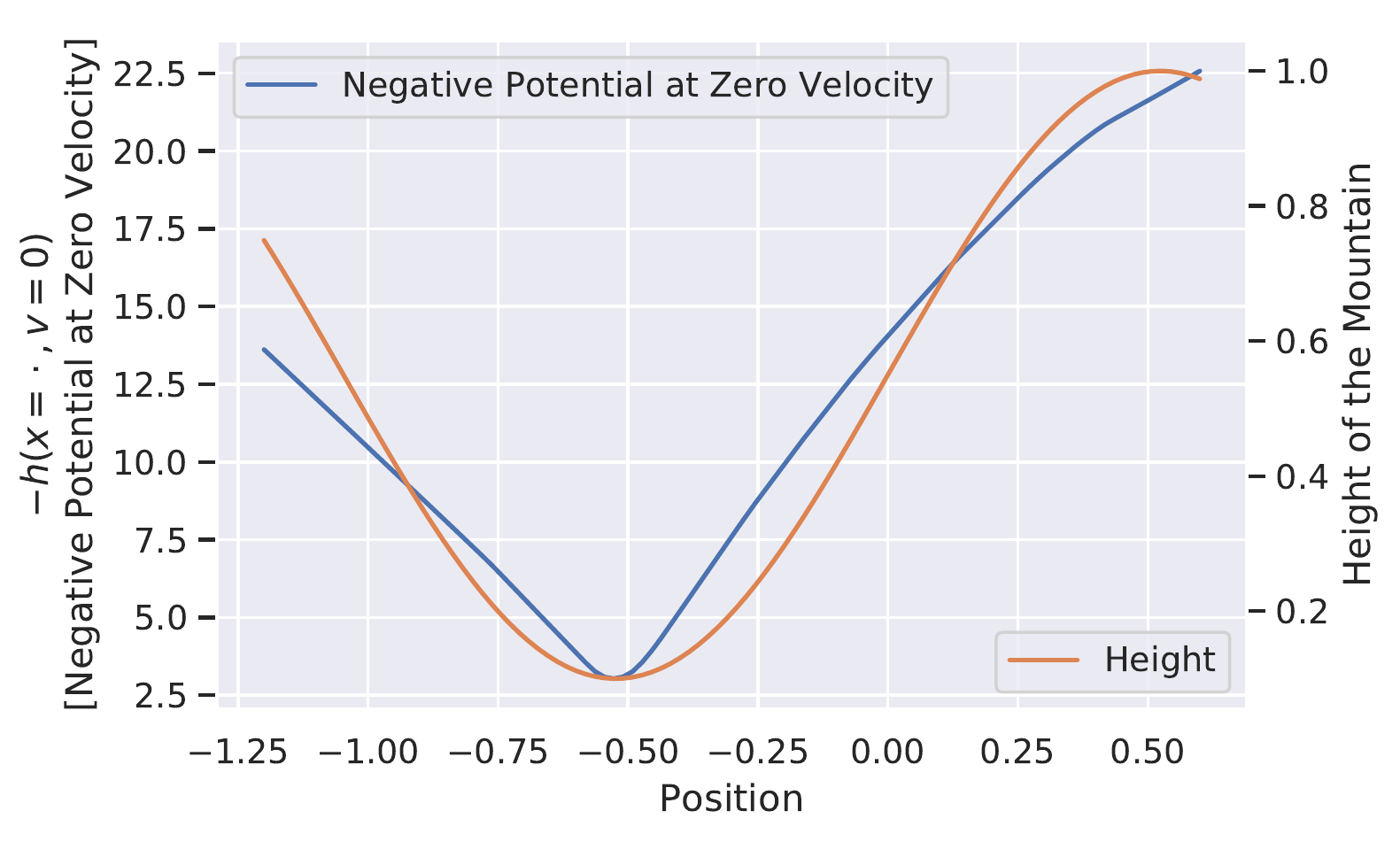}
\caption{\small The $h$-potential (for Mountain Car) at zero-velocity plotted against position. Also plotted (orange) is the height profile of the mountain. \textbf{Gist:} the $h$-potential approximately recovers the height-profile of the mountain with just trajectories from a random policy. \label{fig:mountcar_pot_v_mount}}
\vspace{-13pt}
\end{wrapfigure}

\textbf{Mountain-Car with Friction.\footnote{Experimental details and additional plots can be found in Appendix \ref{app:mountain_car}}} The environment considered shares its dynamics with the well known (continuous) Mountain-Car environment \citep{sutton2011reinforcement}, but with a crucial amendment: the car is subject to friction\footnote{Technically, this is achieved by subtracting a velocity dependent term from the acceleration.}. Friction is required to make the environment dissipative and thereby induce an arrow of time (cf. Section~\ref{sec:interpretation}). Moreover, we initialize the system in a uniform-randomly sampled state to avoid exploration issues. 

\begin{figure}
%\vspace{-15pt}
\centering
\begin{subfigure}{0.49\textwidth} 
\centering
\includegraphics[width=1.\textwidth]{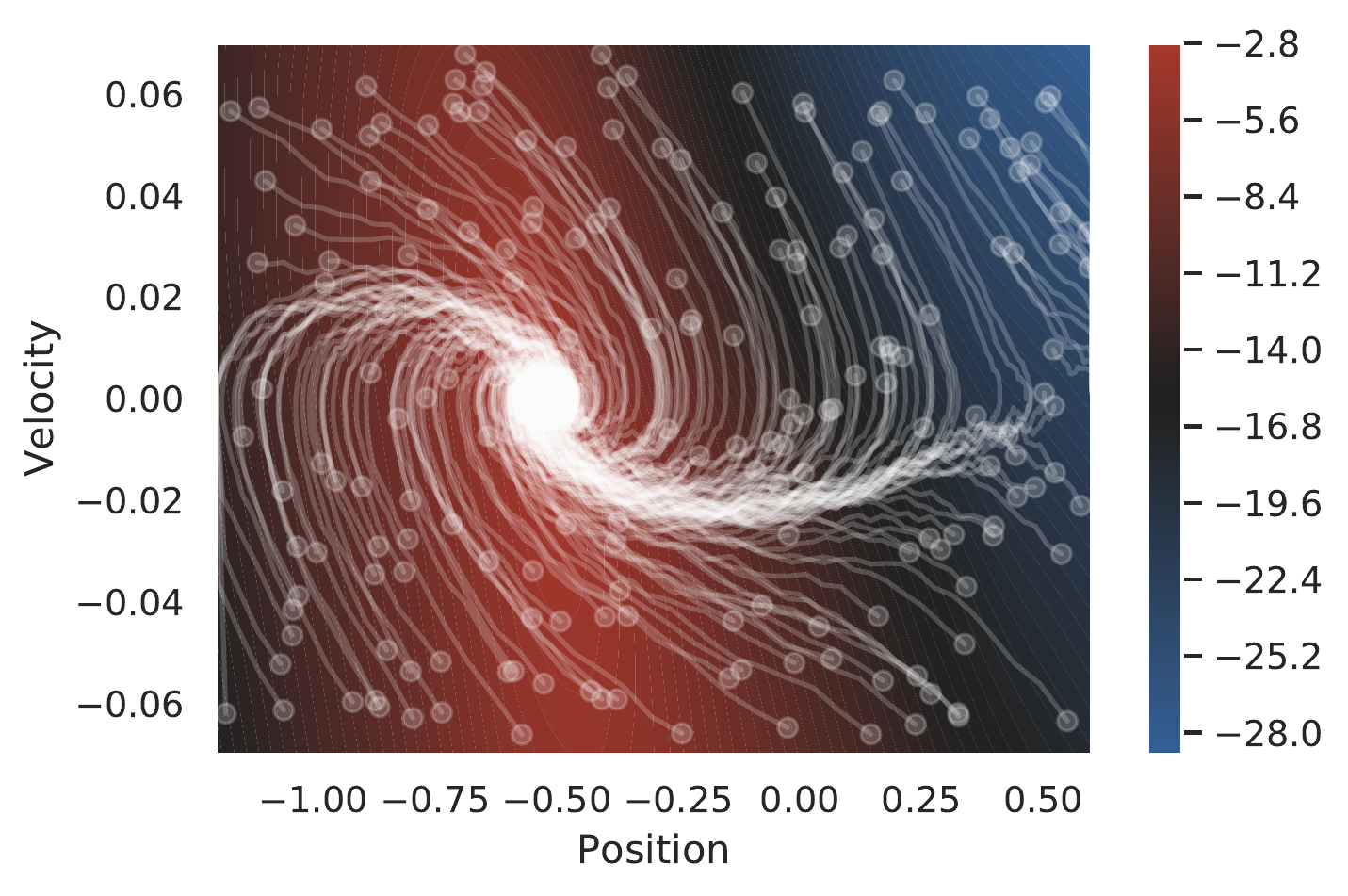}
\caption{\small With friction. \label{fig:mountcar_pot_filled}}
\end{subfigure}\hfill
\begin{subfigure}{0.49\textwidth}
\centering
  \includegraphics[width=1.\textwidth]{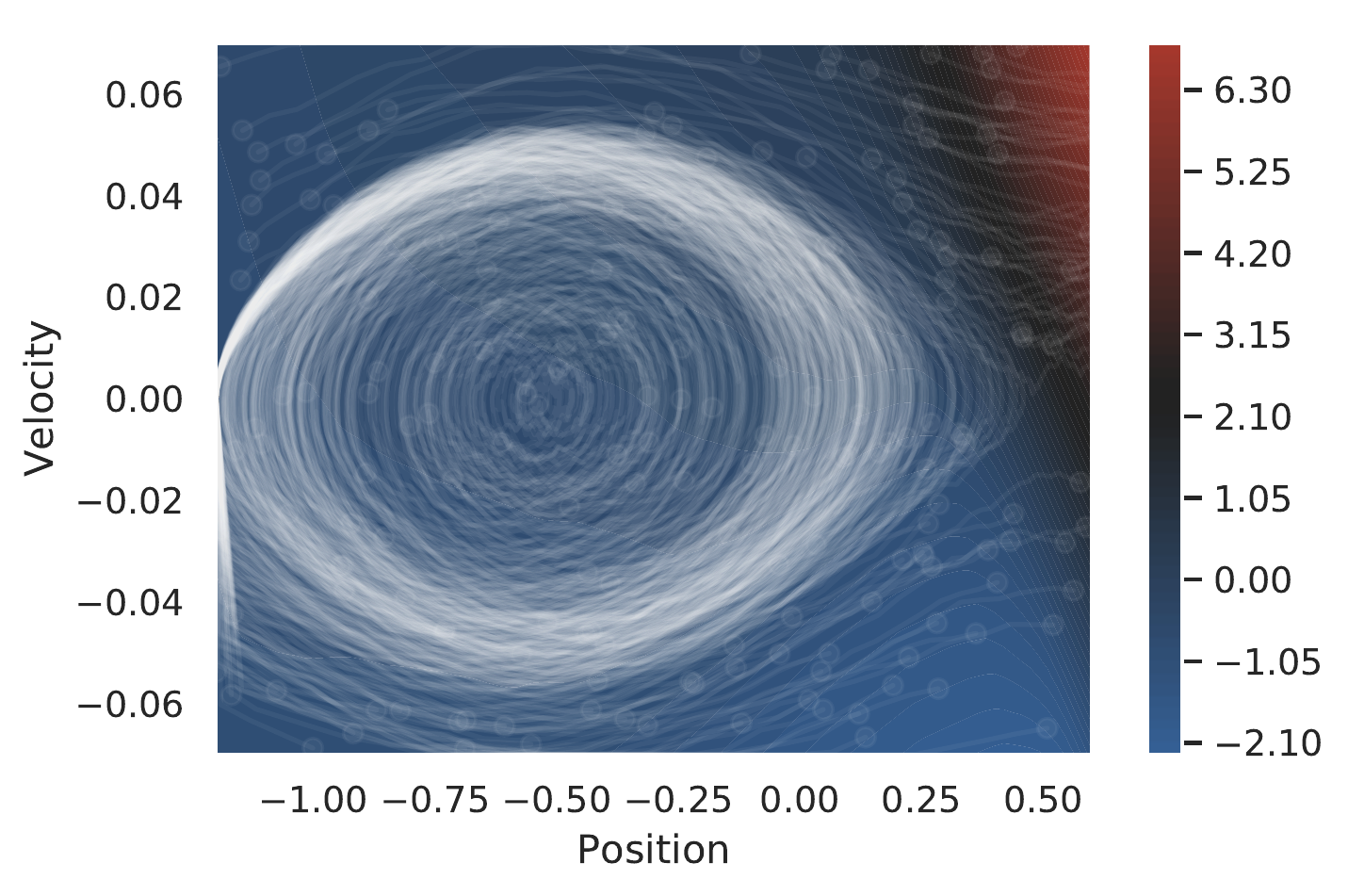}
\caption{\small Without friction. \label{fig:mountcar_pot_filled_nodamp}}
\end{subfigure}
\caption{\small The $h$-potential as a function of state (position and velocity) for (continuous) Mountain-Car with and without friction. The overlay shows random trajectories (emanating from the dots). \textbf{Gist:} with friction, we find that the state with largest $h$ is one where the car is stationary at the bottom of the valley. Without friction, there is no dissipation and the car oscillates up and down the valley. Consequently, we observe that the $h$-potential is constant (up-to edge effects) and thereby uninformative.}
\vspace{-15pt}
\end{figure}

Fig~\ref{fig:mountcar_pot_filled} shows the output of the $h$-potential trained with trajectory regularization overlayed with random trajectories, whereas Fig~\ref{fig:mountcar_pot_v_mount} plots the negative potential at zero-velocity together with the height of the mountain. We not only find that $h$ is at its maximum around the valley, but also that $-h$ at zero velocity largely recovers the terrain just from random trajectories. In addition, we also train the $h$-potential under identical conditions but without friction. The resulting environment is not dissipative, and in Fig~\ref{fig:mountcar_pot_filled_nodamp} we accordingly find that the corresponding $h$-potential is not informative (and an order of magnitude smaller), highlighting the practical importance of dissipation. 

Appendices \ref{app:tomatoworld} and \ref{app:damped_pendulum} present additional experiments. The former shows for a discrete environment that the $h$-potential can be used to derive a reward signal that correlates well with what one might engineer; in the latter, we learn the $h$-potential for an under-damped pendulum. 
\subsection{Connection to Stochastic Processes} \label{sec:ito_sdes}
In this section, we empirically study the link between our method and the theory of stochastic processes. Concretely: our goal is to investigate whether a learned arrow of time behaves \emph{as expected} by comparing it with a known notion of an arrow of time due to \citet*{jordan1998variational}. Experimental details are provided in Appendix~\ref{app:jko}. 

\begin{wrapfigure}{r}{0.48\textwidth}
\vspace{-10pt}
\centering
\includegraphics[width=0.48\textwidth]{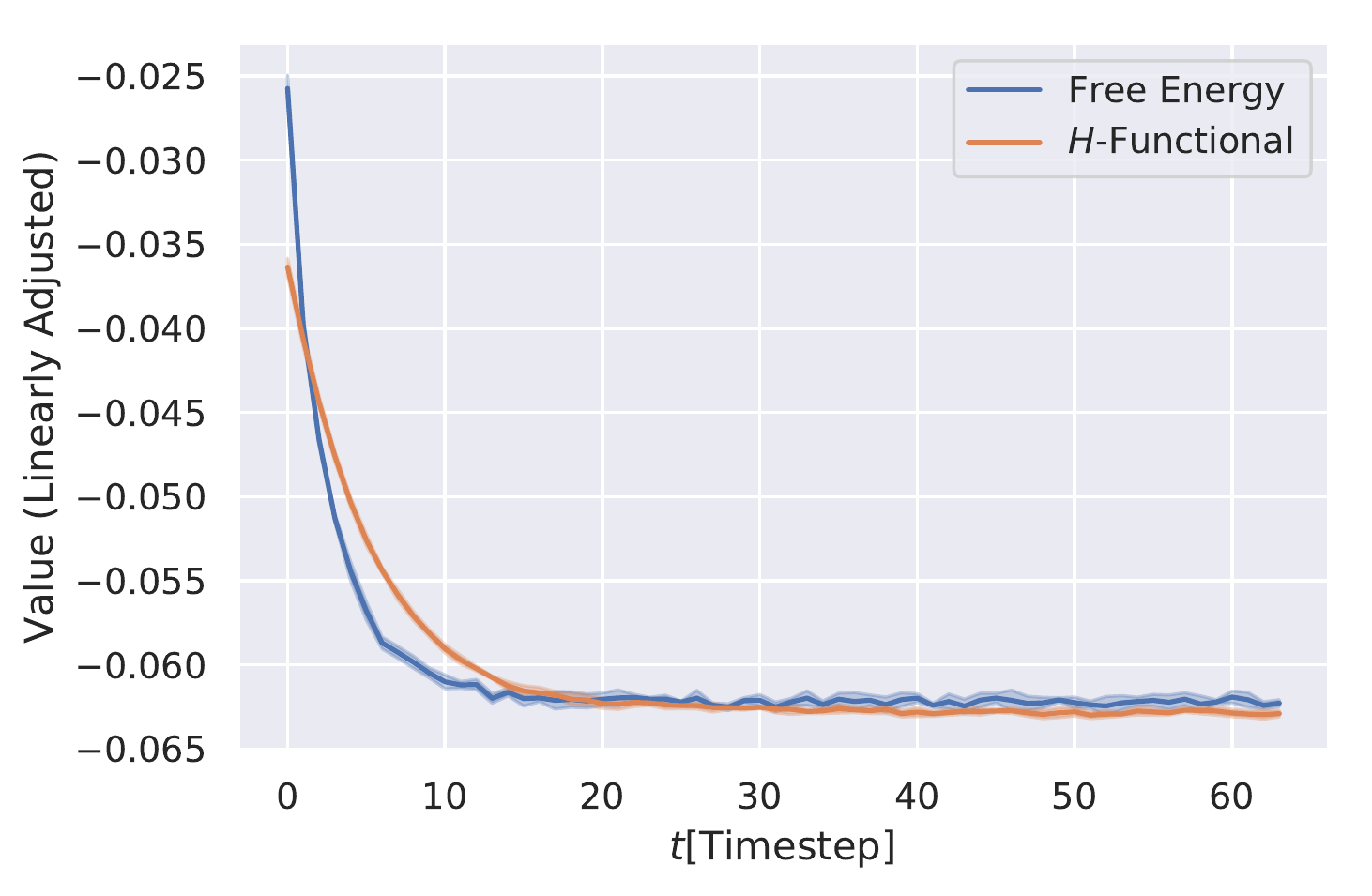}
\caption{\small The \emph{true} arrow of time (the Free-Energy functional, in blue) plotted against the learned arrow of time (the $H$-functional, plotted in orange) after linear scaling and shifting. We find the two to be in good (albeit not perfect) agreement.}\label{fig:rw_lyapunov}
\end{wrapfigure}
Following the notation of \citet{jordan1998variational}, we consider the spatial distribution $\rho(\vx, t)$ at time $t$ of a particle undergoing Brownian motion in the presence of a potential $\Psi$. The Ito stochastic differential equation (associated with the Fokker-Planck-Equation) is given by: 
\beq \label{eq:ito}
d\vX(t) = -\nabla \Psi(\vX(t)) dt + \sqrt{2 \beta^{-1}} d\vW(t)
\eeq
where $\vX(t)$ is the random variable corresponding to the distribution $\rho(\vx, t)$, the initial spatial distribution $\rho^0(\vx) = \rho(\vx, t=0)$ is fixed, $\beta$ is a parameter and $\vW$ is the standard Wiener process (or equivalently, $d\vW$ is uncorrelated white-noise). The celebrated Jordan-Kinderlehrer-Otto result \citep{jordan1998variational} shows that the dynamics of a distribution satisfying the Ito SDE (Eqn~\ref{eq:ito}) has the property that the following Free-Energy functional $F[\rho(\cdot, t)]$ can only decrease with time\footnote{In other words, $F[\rho(\cdot, t)]$ is a Lyapunov functional.}: 
\begin{align}
F[\rho(\cdot, t)] &= \underbrace{\int_{\bR^n} \Psi \rho(\cdot, t) \,\mathbf{dx}}_{E[\rho(\cdot, t)]} + \beta^{-1} \underbrace{\int_{\bR^n} \rho(\cdot, t) \log \rho(\cdot, t) \,\mathbf{dx}}_{-S[\rho(\cdot, t)]} = \E_{\vx \sim \rho(\cdot, t)} [\Psi + \beta^{-1} \log \rho(\cdot, t)]
\end{align}
where $E[\rho]$ is the energy functional and $S[\rho]$ is the (Gibbs-Boltzmann) entropy functional. It follows that the free-energy functional $F$ induces an arrow of time for the stochastic process generated by Eqn~\ref{eq:ito}.

Given that background, the question we now ask is the following: given just samples from the stochastic process $\vX(t)$, how well does a learned $h$-potential agree with the true Free-Energy functional? To answer that, we first define the $H$-functional as:
\beq
H[\rho(\cdot, t)] = -\E_{\vx \sim \rho(\cdot, t)} [h(\vx)]
\eeq
This allows us to compare $H[\rho(\cdot, t)]$ with $F[\rho(\cdot, t)]$, modulo a linear scaling and shift. To that end, we train $h$ with realizations of two-dimensional random walks under an elliptic paraboloid potential $\Psi$. Further, $\E_{\vx \sim \rho(\cdot, t)}[\Psi]$ is estimated via Monte-Carlo sampling, the differential entropy $\E_{\vx \sim \rho(\cdot, t)}[\log \rho(\cdot, t)]$ via a non-parametric estimator \citep{kozachenko1987sample, kraskov2004estimating, gao2015efficient}, and the linear transform coefficients for $H$ via linear regression. Fig~\ref{fig:rw_pot_filled} (in Appendix \ref{app:jko}) plots $h$ as a function of state $\vx \in \bR^2$, whereas Fig~\ref{fig:rw_lyapunov} shows that after appropriate (linear) scaling, the learned $H$ largely agrees with the true $F$. 
\vspace{-5pt}
\section{Conclusion}
\vspace{-5pt}
Over the course of the paper, we addressed the problem of learning the arrow of time in Markov (Decision) Processes. Having formulated an objective (Eqn~\ref{eq:objective_w_reg}) and analyzed the corresponding optimization problem for discrete state-spaces (Section~\ref{sec:theoretical_analysis} and Appendix~\ref{app:theoretical_analysis}), we laid out the fundamental challenges that arise -- namely the presence of \emph{demonic} policies and the requirement that the environment be \emph{dissipative} (Section~\ref{sec:interpretation}). Under appropriate assumptions, we discussed how the arrow of time can be used to measure reachability, detect side-effects and define a curiosity reward (Section~\ref{sec:applications}). Subsequently, we demonstrated the process of learning the arrow of time on a selection of discrete and continuous environments (Section~\ref{sec:experiments}). Finally, we showed for random walks that the learned arrow of time agrees well with the Free-Energy functional, which acts as the \emph{true} arrow of time. Future work could draw connections to algorithmic independence of cause and mechanism \citep{janzing2016algorithmic} and explore applications in causal inference \citep{janzing2010entropy, peters2017elements}. 

\section*{Acknowledgements}
The authors would like to acknowledge Min Lin for the initial discussions, Simon Ramstedt, Zaf Ahmed and Maximilian Puelma Touzel for their feedback on the draft.

\bibliographystyle{unsrtnat}
\bibliography{example_paper}

\begin{thebibliography}{46}
\providecommand{\natexlab}[1]{#1}
\providecommand{\url}[1]{\texttt{#1}}
\expandafter\ifx\csname urlstyle\endcsname\relax
  \providecommand{\doi}[1]{doi: #1}\else
  \providecommand{\doi}{doi: \begingroup \urlstyle{rm}\Url}\fi

\bibitem[Eddington(1929)]{eddington1929nature}
Arthur~Stanley Eddington.
\newblock \emph{The nature of the physical world / by A.S. Eddington}.
\newblock Cambridge University Press Cambridge, England, 1st ed. edition, 1929.

\bibitem[Lyapunov(1892)]{lyapunov1892general}
Aleksandr~Mikhailovich Lyapunov.
\newblock The general problem of the stability of motion.
\newblock \emph{Kharkov Mathematical Society}, 1892.

\bibitem[Prigogine(1978)]{prigogine1978time}
Ilya Prigogine.
\newblock Time, structure, and fluctuations.
\newblock \emph{Science}, 201\penalty0 (4358):\penalty0 777--785, 1978.

\bibitem[Jordan et~al.(1998)Jordan, Kinderlehrer, and
  Otto]{jordan1998variational}
Richard Jordan, David Kinderlehrer, and Felix Otto.
\newblock The variational formulation of the fokker--planck equation.
\newblock \emph{SIAM journal on mathematical analysis}, 29\penalty0
  (1):\penalty0 1--17, 1998.

\bibitem[Crooks(1999)]{crooks1999entropy}
Gavin~E. Crooks.
\newblock Entropy production fluctuation theorem and the nonequilibrium work
  relation for free energy differences.
\newblock \emph{Phys. Rev. E}, 60:\penalty0 2721--2726, Sep 1999.
\newblock \doi{10.1103/PhysRevE.60.2721}.
\newblock URL \url{https://link.aps.org/doi/10.1103/PhysRevE.60.2721}.

\bibitem[Zurek(1989)]{zurek1989algorithmic}
Wojciech~H Zurek.
\newblock Algorithmic randomness and physical entropy.
\newblock \emph{Physical Review A}, 40\penalty0 (8):\penalty0 4731, 1989.

\bibitem[Zurek(1998)]{zurek1998decoherence}
Wojciech~H Zurek.
\newblock Decoherence, chaos, quantum-classical correspondence, and the
  algorithmic arrow of time.
\newblock \emph{Physica Scripta}, 1998\penalty0 (T76):\penalty0 186, 1998.

\bibitem[Janzing et~al.(2016)Janzing, Chaves, and
  Sch{\"o}lkopf]{janzing2016algorithmic}
Dominik Janzing, Rafael Chaves, and Bernhard Sch{\"o}lkopf.
\newblock Algorithmic independence of initial condition and dynamical law in
  thermodynamics and causal inference.
\newblock \emph{New Journal of Physics}, 18\penalty0 (9):\penalty0 093052,
  2016.

\bibitem[Janzing(2010)]{janzing2010entropy}
Dominik Janzing.
\newblock On the entropy production of time series with unidirectional
  linearity.
\newblock \emph{Journal of Statistical Physics}, 138\penalty0 (4-5):\penalty0
  767--779, 2010.

\bibitem[Bauer et~al.(2016)Bauer, Sch{\"o}lkopf, and Peters]{bauer2016arrow}
Stefan Bauer, Bernhard Sch{\"o}lkopf, and Jonas Peters.
\newblock The arrow of time in multivariate time series.
\newblock In \emph{International Conference on Machine Learning}, pages
  2043--2051, 2016.

\bibitem[Seifert(2012)]{seifert2012stochastic}
Udo Seifert.
\newblock Stochastic thermodynamics, fluctuation theorems and molecular
  machines.
\newblock \emph{Reports on progress in physics}, 75\penalty0 (12):\penalty0
  126001, 2012.

\bibitem[Hans et~al.()Hans, Schneega{\ss}, Sch{\"a}fer, and
  Udluft]{hans2008safe}
Alexander Hans, Daniel Schneega{\ss}, Anton~Maximilian Sch{\"a}fer, and Steffen
  Udluft.
\newblock Safe exploration for reinforcement learning.

\bibitem[Moldovan and Abbeel(2012)]{moldovan2012safe}
Teodor~Mihai Moldovan and Pieter Abbeel.
\newblock Safe exploration in markov decision processes.
\newblock \emph{arXiv preprint arXiv:1205.4810}, 2012.

\bibitem[Eysenbach et~al.(2017)Eysenbach, Gu, Ibarz, and
  Levine]{eysenbach2017leave}
Benjamin Eysenbach, Shixiang Gu, Julian Ibarz, and Sergey Levine.
\newblock Leave no trace: Learning to reset for safe and autonomous
  reinforcement learning.
\newblock \emph{arXiv preprint arXiv:1711.06782}, 2017.

\bibitem[Goyal et~al.(2018)Goyal, Brakel, Fedus, Singhal, Lillicrap, Levine,
  Larochelle, and Bengio]{goyal2018recall}
Anirudh Goyal, Philemon Brakel, William Fedus, Soumye Singhal, Timothy
  Lillicrap, Sergey Levine, Hugo Larochelle, and Yoshua Bengio.
\newblock Recall traces: Backtracking models for efficient reinforcement
  learning.
\newblock \emph{arXiv preprint arXiv:1804.00379}, 2018.

\bibitem[Nair et~al.(2018)Nair, Babaeizadeh, Finn, Levine, and
  Kumar]{nair2018time}
Suraj Nair, Mohammad Babaeizadeh, Chelsea Finn, Sergey Levine, and Vikash
  Kumar.
\newblock Time reversal as self-supervision.
\newblock \emph{arXiv preprint arXiv:1810.01128}, 2018.

\bibitem[Amodei et~al.(2016)Amodei, Olah, Steinhardt, Christiano, Schulman, and
  Man{\'e}]{amodei2016concrete}
Dario Amodei, Chris Olah, Jacob Steinhardt, Paul Christiano, John Schulman, and
  Dan Man{\'e}.
\newblock Concrete problems in ai safety.
\newblock \emph{arXiv preprint arXiv:1606.06565}, 2016.

\bibitem[Krakovna et~al.(2018)Krakovna, Orseau, Martic, and
  Legg]{krakovna2018measuring}
Victoria Krakovna, Laurent Orseau, Miljan Martic, and Shane Legg.
\newblock Measuring and avoiding side effects using relative reachability.
\newblock \emph{CoRR}, abs/1806.01186, 2018.
\newblock URL \url{http://arxiv.org/abs/1806.01186}.

\bibitem[Wei et~al.(2018)Wei, Lim, Zisserman, and Freeman]{wei2018learning}
Donglai Wei, Joseph~J Lim, Andrew Zisserman, and William~T Freeman.
\newblock Learning and using the arrow of time.
\newblock In \emph{Proceedings of the IEEE Conference on Computer Vision and
  Pattern Recognition}, pages 8052--8060, 2018.

\bibitem[Pickup et~al.(2014)Pickup, Pan, Wei, Shih, Zhang, Zisserman,
  Scholkopf, and Freeman]{pickup2014seeing}
Lyndsey~C Pickup, Zheng Pan, Donglai Wei, YiChang Shih, Changshui Zhang, Andrew
  Zisserman, Bernhard Scholkopf, and William~T Freeman.
\newblock Seeing the arrow of time.
\newblock In \emph{Proceedings of the IEEE Conference on Computer Vision and
  Pattern Recognition}, pages 2035--2042, 2014.

\bibitem[Li et~al.(2013)Li, Zhang, and Liu]{li2013stability}
Yan Li, Weihai Zhang, and Xikui Liu.
\newblock Stability of nonlinear stochastic discrete-time systems.
\newblock \emph{Journal of Applied Mathematics}, 2013, 2013.

\bibitem[Thomson(1874)]{thomson1874kinetic}
William Thomson.
\newblock Kinetic theory of the dissipation of energy, 1874.

\bibitem[Ng and Russell(2000)]{ng2000algorithms}
Andrew~Y Ng and Stuart Russell.
\newblock Algorithms for inverse reinforcement learning.
\newblock In \emph{in Proc. 17th International Conf. on Machine Learning}.
  Citeseer, 2000.

\bibitem[Munos et~al.(2016)Munos, Stepleton, Harutyunyan, and
  Bellemare]{munos2016safe}
R{\'e}mi Munos, Tom Stepleton, Anna Harutyunyan, and Marc Bellemare.
\newblock Safe and efficient off-policy reinforcement learning.
\newblock In \emph{Advances in Neural Information Processing Systems}, pages
  1054--1062, 2016.

\bibitem[Ha and Schmidhuber(2018)]{ha2018world}
David Ha and J{\"u}rgen Schmidhuber.
\newblock World models.
\newblock \emph{arXiv preprint arXiv:1803.10122}, 2018.

\bibitem[Bunimovich(2007)]{bunimovich2007billiards}
L.~Bunimovich.
\newblock {D}ynamical billiards.
\newblock \emph{Scholarpedia}, 2\penalty0 (8):\penalty0 1813, 2007.
\newblock \doi{10.4249/scholarpedia.1813}.
\newblock revision \#91212.

\bibitem[Willems(1972)]{willems1972dissipative}
Jan~C Willems.
\newblock Dissipative dynamical systems part i: General theory.
\newblock \emph{Archive for rational mechanics and analysis}, 45\penalty0
  (5):\penalty0 321--351, 1972.

\bibitem[McDonald(2015)]{mcdonald2015damped}
Kirk~T McDonald.
\newblock A damped oscillator as a hamiltonian system.
\newblock 2015.

\bibitem[Savinov et~al.(2018)Savinov, Raichuk, Marinier, Vincent, Pollefeys,
  Lillicrap, and Gelly]{savinov2018episodic}
Nikolay Savinov, Anton Raichuk, Rapha{\"e}l Marinier, Damien Vincent, Marc
  Pollefeys, Timothy Lillicrap, and Sylvain Gelly.
\newblock Episodic curiosity through reachability.
\newblock \emph{arXiv preprint arXiv:1810.02274}, 2018.

\bibitem[Leike et~al.(2017)Leike, Martic, Krakovna, Ortega, Everitt, Lefrancq,
  Orseau, and Legg]{leike2017ai}
Jan Leike, Miljan Martic, Victoria Krakovna, Pedro~A Ortega, Tom Everitt,
  Andrew Lefrancq, Laurent Orseau, and Shane Legg.
\newblock Ai safety gridworlds.
\newblock \emph{arXiv preprint arXiv:1711.09883}, 2017.

\bibitem[Armstrong and Levinstein(2017)]{armstrong2017low}
Stuart Armstrong and Benjamin Levinstein.
\newblock Low impact artificial intelligences.
\newblock \emph{arXiv preprint arXiv:1705.10720}, 2017.

\bibitem[Schmidhuber(2010)]{schmidhuber2010formal}
J{\"u}rgen Schmidhuber.
\newblock Formal theory of creativity, fun, and intrinsic motivation
  (1990--2010).
\newblock \emph{IEEE Transactions on Autonomous Mental Development}, 2\penalty0
  (3):\penalty0 230--247, 2010.

\bibitem[Chentanez et~al.(2005)Chentanez, Barto, and
  Singh]{chentanez2005intrinsically}
Nuttapong Chentanez, Andrew~G Barto, and Satinder~P Singh.
\newblock Intrinsically motivated reinforcement learning.
\newblock In \emph{Advances in neural information processing systems}, pages
  1281--1288, 2005.

\bibitem[Pathak et~al.(2017)Pathak, Agrawal, Efros, and
  Darrell]{pathakICMl17curiosity}
Deepak Pathak, Pulkit Agrawal, Alexei~A. Efros, and Trevor Darrell.
\newblock Curiosity-driven exploration by self-supervised prediction.
\newblock In \emph{ICML}, 2017.

\bibitem[Burda et~al.(2018)Burda, Edwards, Pathak, Storkey, Darrell, and
  Efros]{burda2018large}
Yuri Burda, Harri Edwards, Deepak Pathak, Amos Storkey, Trevor Darrell, and
  Alexei~A Efros.
\newblock Large-scale study of curiosity-driven learning.
\newblock \emph{arXiv preprint arXiv:1808.04355}, 2018.

\bibitem[Schrader(2018)]{schrader2018sokoban}
Max-Philipp~B. Schrader.
\newblock gym-sokoban.
\newblock \url{https://github.com/mpSchrader/gym-sokoban}, 2018.

\bibitem[Sutton and Barto(2011)]{sutton2011reinforcement}
Richard~S Sutton and Andrew~G Barto.
\newblock Reinforcement learning: An introduction.
\newblock 2011.

\bibitem[Kozachenko and Leonenko(1987)]{kozachenko1987sample}
LF~Kozachenko and Nikolai~N Leonenko.
\newblock Sample estimate of the entropy of a random vector.
\newblock \emph{Problemy Peredachi Informatsii}, 23\penalty0 (2):\penalty0
  9--16, 1987.

\bibitem[Kraskov et~al.(2004)Kraskov, St{\"o}gbauer, and
  Grassberger]{kraskov2004estimating}
Alexander Kraskov, Harald St{\"o}gbauer, and Peter Grassberger.
\newblock Estimating mutual information.
\newblock \emph{Physical review E}, 69\penalty0 (6):\penalty0 066138, 2004.

\bibitem[Gao et~al.(2015)Gao, Ver~Steeg, and Galstyan]{gao2015efficient}
Shuyang Gao, Greg Ver~Steeg, and Aram Galstyan.
\newblock Efficient estimation of mutual information for strongly dependent
  variables.
\newblock In \emph{Artificial Intelligence and Statistics}, pages 277--286,
  2015.

\bibitem[Peters et~al.(2017)Peters, Janzing, and
  Sch{\"o}lkopf]{peters2017elements}
Jonas Peters, Dominik Janzing, and Bernhard Sch{\"o}lkopf.
\newblock \emph{Elements of causal inference: foundations and learning
  algorithms}.
\newblock MIT press, 2017.

\bibitem[Shangtong(2018)]{deeprl}
Zhang Shangtong.
\newblock Modularized implementation of deep rl algorithms in pytorch.
\newblock \url{https://github.com/ShangtongZhang/DeepRL}, 2018.

\bibitem[Van~Hasselt et~al.(2016)Van~Hasselt, Guez, and Silver]{van2016deep}
Hado Van~Hasselt, Arthur Guez, and David Silver.
\newblock Deep reinforcement learning with double q-learning.
\newblock In \emph{Thirtieth AAAI Conference on Artificial Intelligence}, 2016.

\bibitem[Wang et~al.(2015)Wang, Schaul, Hessel, Van~Hasselt, Lanctot, and
  De~Freitas]{wang2015dueling}
Ziyu Wang, Tom Schaul, Matteo Hessel, Hado Van~Hasselt, Marc Lanctot, and Nando
  De~Freitas.
\newblock Dueling network architectures for deep reinforcement learning.
\newblock \emph{arXiv preprint arXiv:1511.06581}, 2015.

\bibitem[Savitzky and Golay(1964)]{savitzky1964smoothing}
Abraham Savitzky and Marcel~JE Golay.
\newblock Smoothing and differentiation of data by simplified least squares
  procedures.
\newblock \emph{Analytical chemistry}, 36\penalty0 (8):\penalty0 1627--1639,
  1964.

\bibitem[Brockman et~al.(2016)Brockman, Cheung, Pettersson, Schneider,
  Schulman, Tang, and Zaremba]{brockman2016openai}
Greg Brockman, Vicki Cheung, Ludwig Pettersson, Jonas Schneider, John Schulman,
  Jie Tang, and Wojciech Zaremba.
\newblock Openai gym.
\newblock \emph{arXiv preprint arXiv:1606.01540}, 2016.

\end{thebibliography}
\clearpage
\newpage
\appendix
\begin{appendix}
\section{Theoretical Analysis} \label{app:theoretical_analysis}
In this section, we present a theoretical analysis of the optimization problem formulated in Eqn~\ref{eq:objective_w_reg}, and analytically evaluate the result for a few toy Markov processes to validate that the resulting solutions are indeed consistent with intuition. To simplify the exposition, we consider the discrete case where the state space $\cS$ of the MDP is finite. 

Consider a discrete Markov chain with enumerable states $s_i \in \cS$. At an arbitrary (but given) time-step $t$, we let $p^t_i = p(s_t = s_i)$ denote the probability that the Markov chain is in state $s_i$, and $\vp^t$ the corresponding vector (over states). With $T_{ij}$ we denote the probability of the Markov chain transitioning from state $s_i$ to $s_j$ under some policy $\pi$, i.e.\ $T_{ij} = p_{\pi}(s_{t + 1} = s_j | s_{t} = s_i)$. One has the transition rule: 
\beq \label{eq:markov_chain_transition}
\vp^{t + 1} = \vp^t T \qquad \vp^t = \vp^0 T^t
\eeq
where $T^t$ is the $t$-th matrix power of $T$. Now, we let $h_i$ denote the value $h_{\pi}$ takes at state $s_i$, i.e. $h_i = h_{\pi}(s_i)$, and the corresponding vector (over states) becomes $\vh$. This reduces the expectation of the function (now a vector) $\vh$ w.r.t any state distribution (now also a vector) $\vp$ to the scalar product $\vp \cdot \vh$. In matrix notation, the optimization problem in Eqn~\ref{eq:objective_w_reg} simplifies to: 

\beq \label{eq:objective_discrete}
\arg\max_{\vh} \frac{1}{N} \sum_{t = 0}^{N - 1}\left[\vp^t T \vh - \vp^t \cdot \vh \right] + \lambda \cT(\vh)
\eeq
For certain $\cT$, the discrete problem in Eqn~\ref{eq:objective_discrete} can be handled analytically. We consider two candidates for $\cT$, the first being the norm of $\vh$, and the second one being the norm of change in $h_i$, in expectation along a trajectory. 

\begin{proposition} \label{prop:l2_reg}
If $\cT(\vh) = -(2N)^{-1} \|\vh\|^2$, the solution to the optimization problem in Eqn~\ref{eq:objective_discrete} is given by: 
\beq \label{eq:discrete_solution}
\vh = \frac{\vp^0 T^N - \vp^0}{\lambda}
\eeq
\end{proposition}
\begin{proof}
First, note that the objective in Eqn~\ref{eq:objective_discrete} becomes: 
\beq \label{eq:objective_l2}
\cL[\vh] = \frac{1}{N} \sum_{t = 0}^{N - 1}\left[\vp^t T \vh - \vp^t \cdot \vh \right] - \frac{1}{2N} \|\vh\|^2
\eeq
To solve the maximization problem, we must differentiate $\cL$ w.r.t. its argument $\vh$, and set the resulting expression to zero. This yields: 
\beq
\nabla_{\vh} \cL = \frac{1}{N} \left[\sum_{t=0}^{N - 1} (\vp^t T - \vp^t) - \lambda \vh \right] = 0
\eeq
Now, the summation (over $t$) is telescoping, and evaluates to $\vp^{N - 1} T - \vp^0$. Substituting $\vp^{N - 1}$ with the corresponding expression from Eqn~\ref{eq:markov_chain_transition} and solving for $\vh$, we obtain Eqn~\ref{eq:discrete_solution}. 
\end{proof}

Proposition~\ref{prop:l2_reg} has an interesting implication: if the Markov chain is initialized at equilibrium, i.e. if $\vp^0 = \vp^0 T$, we obtain that $\vh = \mathbf{0}$ identically. Given the above, we may now consider simple Markov chains to explore the implications of Eqn~\ref{eq:discrete_solution}.
\begin{example} \label{exp:two_var_mdp}
Consider a Markov chain with two states and reversible transitions, parameterized by $\alpha \in [0, 1]$ such that $T_{11} = T_{21} = 1 - \alpha$ and $T_{12} = T_{22} = \alpha$. If $\vp^0 = (\nicefrac{1}{2}, \nicefrac{1}{2})$, one obtains: 
\beq \label{eq:two_var_markov_l2}
\vh \propto (-\gamma, \gamma)
\eeq
where $\gamma = \alpha - \nicefrac{1}{2}$. To see how, consider that for all $N > 0$, one obtains $\vp^{0} T^N = (1 - \alpha, \alpha)$. Together with Proposition~\ref{prop:l2_reg}, Eqn~\ref{eq:two_var_markov_l2} follows. 
\end{example}

The above example illustrates two things. On the one hand, if $\alpha = \nicefrac{1}{2}$, one obtains a Markov chain with \emph{perfect reversibility}, i.e. the transition $s_1 \to s_2$ is equally as likely as the transition $s_2 \to s_1$. In this case, one indeed obtains $h(s_1) = h(s_2) = 0$, as mentioned above. On the other hand, if one sets $\alpha = 1$, the transition from $s_2 \to s_1$ is never sampled, and that from $s_1 \to s_2$ is irreversible; consequently, $h(s_2) - h(s_1)$ takes the largest value possible. Now, while this aligns well with our intuition, the following example exposes a weakness of the L2-norm-penalty used in Proposition~\ref{prop:l2_reg}. 

\begin{example} \label{exp:four_var_mdp}
Consider two Markov chains, both always initialized at $s_1$. For the first Markov chain, the dynamics admits the following transitions: $s_1 \to s_2 \to s_3 \to s_4$, whereas for the second chain, one has $s_1 \to s_3 \to s_2 \to s_4$. Now, for both chains and $N \ge 4$, it's easy to see that $(\vp^0 T^N)_i = 1$ if $i = 4$, but $0$ otherwise. From Eqn~\ref{eq:discrete_solution}, one obtains: 
\beq \label{eq:four_var_markov_l2}
\vh \propto (-1, 0, 0, 1)
\eeq
\end{example}
The solution for $h$ given by Eqn~\ref{eq:four_var_markov_l2} indeed increases (non-strictly) monotonously with timestep. However, we obtain $h(s_2) = h(s_3) = 0$ for both Markov chains. In particular, $h$ does not increase between the $s_2 \to s_3$ transition in the former and the $s_3 \to s_2$ transition in the latter, even though both transitions are irreversible. It is in general apparent from \ref{prop:l2_reg} that the solution for $h$ depends only on the initial and final state distribution, and not the intermediate trajectory. 

Now, consider the following regularizer that penalizes not just the function norm, but the change in $h$ in expectation along trajectories: 
\beq \label{eq:traj_reg}
\cT(\vh) = -\frac{1}{2N} \sum_{t=0}^{N - 1} (\vp^t T \vh - \vp^t \cdot \vh)^2 - \frac{\omega}{2N} \|\vh\|^2
\eeq 
where $\omega$ is the relative weight of the L2 regularizer. This leads to the result:
\begin{proposition}
The solution to the optimization problem in Eqn~\ref{eq:objective_discrete} with the regularizer in Eqn~\ref{eq:traj_reg} is the solution to the following matrix-equation: 
\beq \label{eq:discrete_solution_traj_reg}
\sum_{t=0}^{N-1} \vp^0(T^{t + 1} - T^t)\vh \, \vp^0 (T^{t + 1} - T^t) + \omega \vh = \frac{\vp^0 T^N - \vp^0}{2\lambda}
\eeq
\end{proposition}
\begin{proof}
Analogous to Eqn~\ref{eq:objective_l2}, we may write the objective in Eqn~\ref{eq:objective_discrete} as (by substituting Eqn~\ref{eq:traj_reg} in Eqn~\ref{eq:objective_discrete}):
\beq \label{eq:objective_traj_reg}
\cL[\vh] = \frac{1}{N} \sum_{t = 0}^{N - 1}\left[\vp^t T \vh - \vp^t \cdot \vh \right] - \frac{\lambda}{2N} \sum_{t=0}^{N - 1} (\vp^t T \vh - \vp^t \cdot \vh)^2 - \frac{\lambda \omega}{2N} \|\vh\|^2
\eeq
Like in Proposition~\ref{prop:l2_reg}, we maximize it by setting the gradient of $\cL$ w.r.t. $\vh$ to zero. This yields: 
\beq \label{eq:grad_obj_traj_reg}
\nabla_{\vh} \cL = \frac{1}{N} \left[\sum_{t=0}^{N - 1} (\vp^t T - \vp^t) - \frac{\lambda}{2} \nabla_{\vh} \sum_{t=0}^{N - 1} (\vp^t T \vh - \vp^t \cdot \vh)^2 - \omega\lambda \vh \right] = 0
\eeq
The first term in the RHS is again a telescoping sum; it evaluates to: $\vp^0 T^N - \vp^0$ (cf. proof of Proposition~\ref{prop:l2_reg}). The second term can be expressed as (with $I$ as the identity matrix): 
\begin{align}
\frac{\lambda}{2} \nabla_{\vh} \sum_{t=0}^{N - 1} (\vp^t T \vh - \vp^t \cdot \vh)^2 &= 
\frac{\lambda}{2} \sum_{t=0}^{N - 1} \nabla_{\vh} (\vp^t (T - I) \vh)^2 \\
&= \lambda \sum_{t=0}^{N - 1} (\vp^t (T - I) \vh)(\vp^t (T - I)) \\
&= \lambda \sum_{t=0}^{N-1} \vp^0(T^{t + 1} - T^t)\vh \, \vp^0 (T^{t + 1} - T^t)
\end{align}
where the last equality follows from Eqn~\ref{eq:markov_chain_transition}. Substituting the above in Eqn~\ref{eq:grad_obj_traj_reg} and rearranging terms yields Eqn~\ref{eq:discrete_solution_traj_reg}. 
\end{proof}
While Eqn~\ref{eq:discrete_solution_traj_reg} does not yield an explicit expression for $\vh$, it is sufficient for analysing individual cases considered in Examples \ref{exp:two_var_mdp} and \ref{exp:four_var_mdp}. 

\begin{example} \label{exp:two_var_mdp_tr}
Consider the two-state Markov chain in Example~\ref{exp:two_var_mdp} and the associated transition matrix $T$ and initial state distribution $\vp^0 = (\nicefrac{1}{2}, \nicefrac{1}{2})$. Using the regularization scheme in Eqn~\ref{eq:traj_reg} and the associated solution Eqn~\ref{eq:discrete_solution_traj_reg}, one obtains: 
\beq
\vh = (-\tilde \gamma, \tilde \gamma)
\eeq
where:
\beq \label{eq:two_var_markov_traj_reg}
\tilde \gamma = \frac{2\alpha - 1}{\lambda (4 \alpha^2 - 4 \alpha + 2\omega + 1)}
\eeq
To obtain this result\footnote{Interested readers may refer to the attached SymPy computation.}, we use that $T^t = T$ for all $t \ge 1$ and truncate the sum without loss of generality at $N = 1$. 
\end{example}

Like in Example~\ref{exp:two_var_mdp}, we observe $h(s_1) = h(s_2) = 0$ if $\alpha = \nicefrac{1}{2}$ for all $\omega > 0$ (i.e. at equilibrium). In addition, if $\omega \ge \nicefrac{1}{2}$, it can be shown that $h(s_2) - h(s_1)$ increases monotonously with $\alpha$ and takes the largest possible value at $\alpha = 1$. We therefore find that for the simple two-state Markov chain of Example~\ref{exp:two_var_mdp}, the regularization in Eqn~\ref{eq:traj_reg} indeed leads to intuitive behaviour for the respective solution $\vh$. Now:

\begin{example} \label{exp:four_var_mdp_tr}
Consider the four-state Markov chain with transitions $s_1 \to s_2 \to s_3 \to s_4$ and the corresponding transition matrix $T$, where $T_{12} = T_{23} = T_{34} = T_{44} = 1$, $T_{ij} = 0$ for all other $i, j$. Set $\vp^0 = (1, 0, 0, 0)$, i.e. the chain is always initialized at $s_1$. Now, the summation over $t$ in Eqn~\ref{eq:discrete_solution_traj_reg} can be truncated w.l.o.g when $N = 4$, given that $T^{t + 1} = T^t$ for all $t \ge 3$. At $\omega = 0$, one solution is: 
\beq
\vh \propto (-\nicefrac{3}{2}, -\nicefrac{1}{2}, \nicefrac{1}{2}, \nicefrac{3}{2})
\eeq
\end{example}
Further, for all finite $\omega$, one obtains $h(s_1) < h(s_2) < h(s_3) < h(s_4)$, where the inequality is strict. This is unlike Eqn~\ref{eq:four_var_markov_l2} where $h(s_2) = h(s_3)$, and consistent with the intuitive expectation that the arrow of time must increase along irreversible transitions. 

In conclusion: we find that the functional objective defined in Eqn~\ref{eq:objective_w_reg} may indeed lead to analytical solutions that are consistent with the notion of an arrow of time in certain toy Markov chains. However, in most interesting real world environments, the transition model $T$ is not known and or or the number of states is infeasibly large, rendering an analytic solution intractable. In such cases, as we see in Section~\ref{sec:experiments}, it is possible to parameterize $h$ as a neural network and train the resulting model with stochastic gradient descent to optimize the functional objective defined in Eqn~\ref{eq:objective_w_reg}. 

\section{Algorithm} \label{app:algorithm}

\begin{algorithm}[H]
\caption{Training the $h$-Potential} \label{alg:main}
\begin{algorithmic} [1]
\REQUIRE Environment \texttt{Env}, random policy $\pi_{\sharp}$, trajectory buffer \texttt{B}
\REQUIRE Model $h^{\theta}$, regularizer $\cT$, optimizer.
\FOR{$k = 1...M$}
\STATE $\texttt{B[k, :]} \gets (s_0, ..., s_{N}) \sim \texttt{Env}[\pi_{\sharp}]$ \COMMENT{Sample a trajectory of length $N$ with the random policy and write to $k$-th position in the buffer.}
\ENDFOR
\LOOP 
\STATE Sample trajectory index $k \sim \{1, ..., M\}$ and time-step $t \sim \{0, ..., N-1\}$. \COMMENT{In general, one may sample multiple $k$'s and $t$'s for a larger mini-batch.}
\STATE Fetch states $s_t \gets \texttt{B[k, t]}$ and $s_{t + 1} \gets \texttt{B[k, t + 1]}$ from buffer. 
\STATE Compute loss as $L(\theta) = -[h^{\theta}(s_{t + 1}) - h^{\theta}(s_{t})]$. 
\IF{using trajectory regularizer}
\STATE Compute regularizer term as $[h^{\theta}(s_{t + 1}) - h^{\theta}(s_t)]^2$ and add to $L(\theta)$. 
\ELSE
\STATE Apply the regularizer as required. If early-stopping, break out of the loop if necessary. 
\ENDIF
\STATE Compute parameter gradients $\nabla_{\theta}L(\theta)$ and update parameters with the optimizer. 
\ENDLOOP
\end{algorithmic}
\end{algorithm}

\section{Experimental Details} \label{app:experiments}
All experiments were run on a workstation with 40 cores, 256 GB RAM and 2 nVidia GTX 1080Ti.
\subsection{Discrete Environments}
\subsubsection{2D World with Vases} \label{app:vaseworld}
\begin{figure}
%\vspace{-15pt}
\centering
\begin{subfigure}{0.5\textwidth}
\centering
  \includegraphics[width=1.\textwidth]{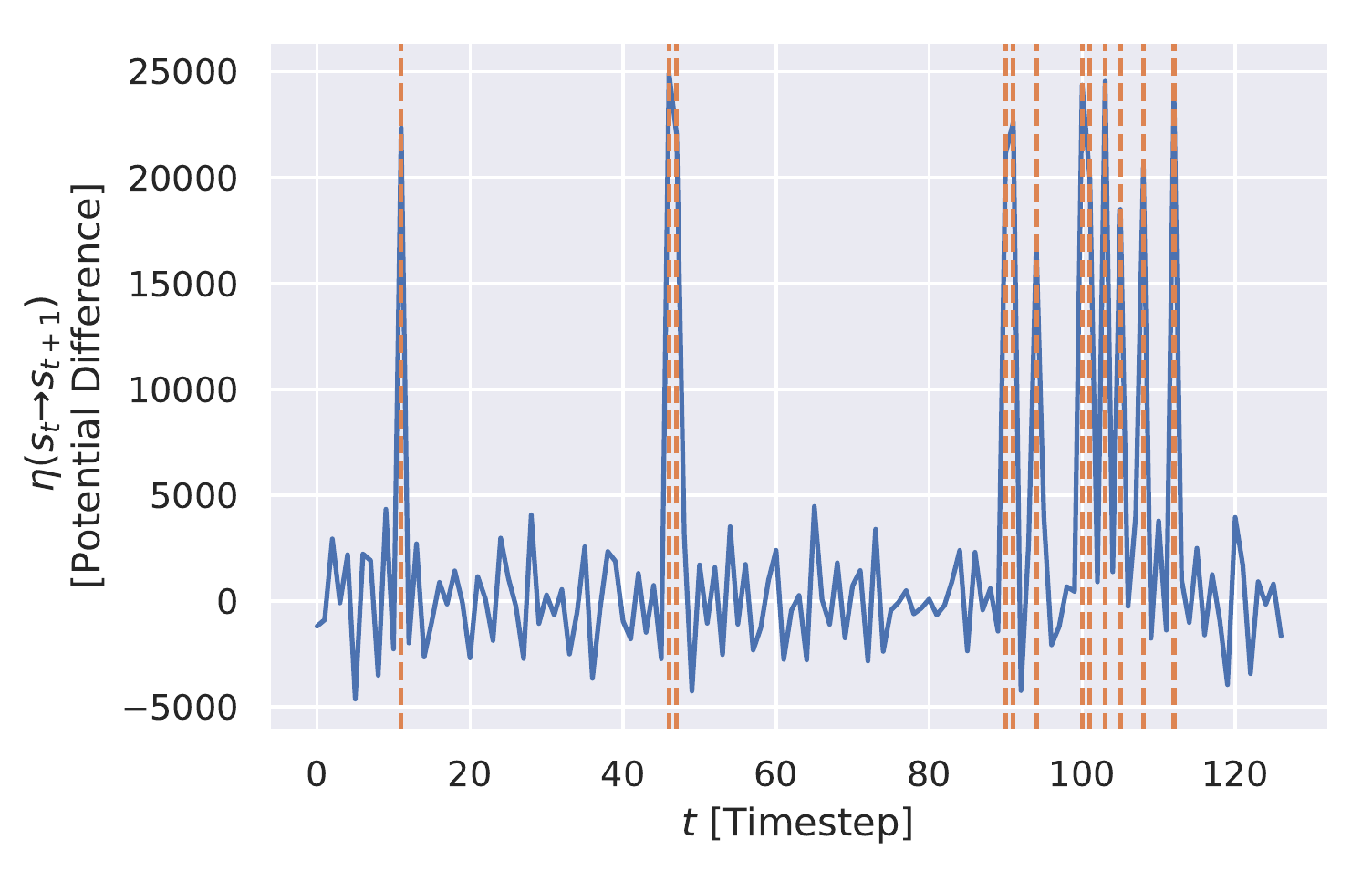}
\caption{\small TV-Noise \label{fig:vaseworld_dpot_tvnoise}}
\end{subfigure}\hfill
\begin{subfigure}{0.5\textwidth} 
\centering
\includegraphics[width=1.\textwidth]{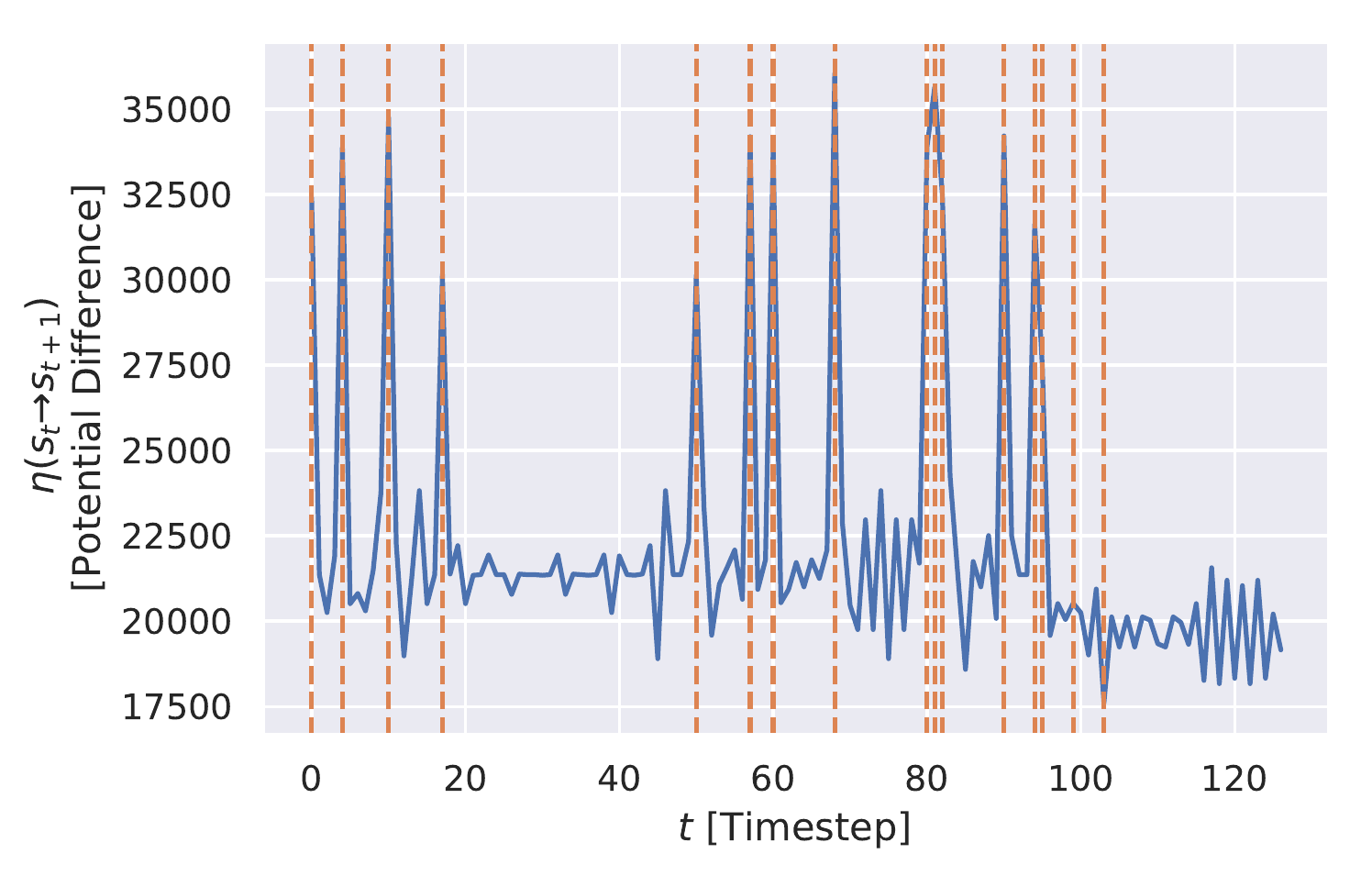}
\caption{\small Causal Noise \label{fig:vaseworld_dpot_causalnoise}}
\end{subfigure}
\caption{\small The potential difference $\eta$ plotted along trajectories, where the state-space is augmented with temporally uncorrelated (TV-) and correlated (causal) noise. The dashed vertical lines indicate time-steps where a vase is broken. \textbf{Gist:} while our method is fairly robust to TV-noise, it might get distracted by causal noise. \label{fig:vaseworld_dpots_noise}}
\end{figure}
\begin{wrapfigure}{r}{0.48\textwidth}
\centering
\includegraphics[width=0.48\textwidth]{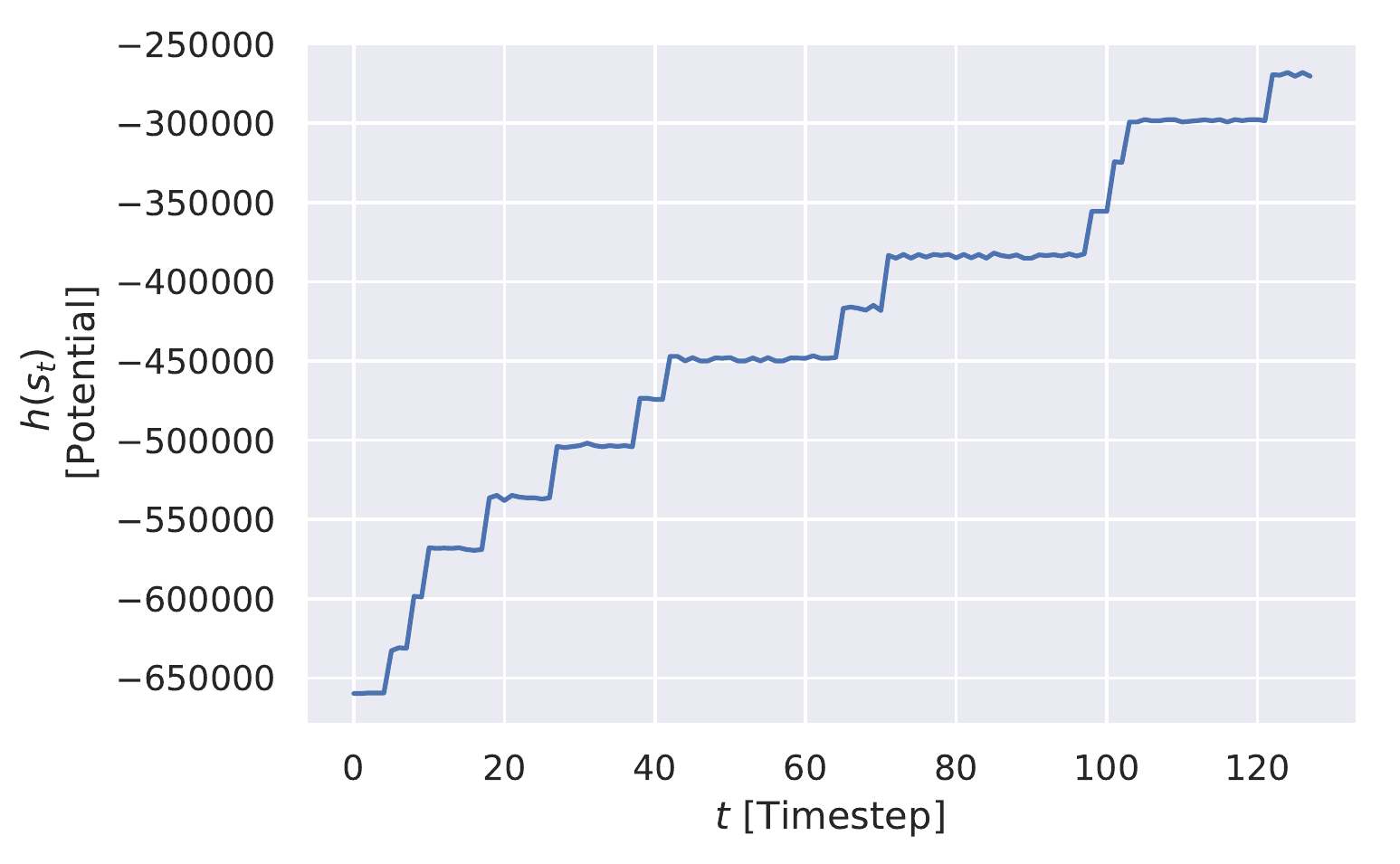}
\caption{\small The $h$-potential along a trajectory sampled from a random policy. \textbf{Gist:} The $h$-potential increases step-wise along the trajectory every time an agent (irreversibly) breaks a vase. It remains constant as the agent (reversibly) moves about.}\label{fig:vaseworld_pot_nonoise}
\vspace{-10pt}
\end{wrapfigure}
The environment state comprises three $7 \times 7$ binary images (corresponding to agent, vases and goal), and the vases appear in a different arrangement every time the environment is reset. The probability of sampling a vase at any given position is set to \nicefrac{1}{2}. 

We use a two-layer deep and 256-unit wide ReLU network to parameterize the $h$-potential. It is trained on 4096 trajectories of length 128 for 10000 iterations of stochastic gradient descent with Adam optimizer (learning rate: 0.0001). The batch-size is set to 128, and we use a weight decay of 0.005 to regularize the model. We use a validation trajectory to generate the plots in Fig~\ref{fig:vaseworld_pot_nonoise} and \ref{fig:vaseworld_dpot_nonoise}. Moreover, Fig~\ref{fig:vaseworld_pothist_nonoise} shows histograms of the values taken by $h$ at various time-steps along the trajectory. We learn that $h$ takes on larger values (on average) as $t$ increases. 

To test the robustness of our method, we conduct experiments where the environment state is augmented with one of: (a) a $7 \times 7$ image with uniform-randomly sampled pixel values (\emph{TV-noise}) and (b) a $7 \times 7$ image where every pixel takes the value $\nicefrac{t}{128}$, where $t$ is the time-step\footnote{Recall that the trajectory length is set to $128$.} of the corresponding state (\emph{Causal Noise}). Fig~\ref{fig:vaseworld_dpot_tvnoise} and \ref{fig:vaseworld_dpot_causalnoise} plot the corresponding $\eta = \Delta h$ along randomly sampled trajectories. 

Given a learned arrow of time, we now present an experiment where we use it to derive a safe-exploration penalty (in addition to the environment reward). To that end, we now consider the situation where the agent's policy is not random, but specialized to reach the goal state (from its current state). For both the baseline and the safe agents, every action is rewarded with the change in Manhattan norm of the agent's position to that of the goal -- i.e. an action that moves the agent closer to the goal is rewarded $+1$, one that moves it farther away from the goal is penalized $-1$, and one that keeps the distance unchanged is neither penalized nor rewarded ($0$). Further, every step is penalized by $-0.1$ (so as to keep the trajectories short), and exceeding the available time limit ($30$ steps) incurs a termination penalty ($-10$). In addition, the reward function of the safe agent is augmented with the reachability, i.e. it takes the form described in Eqn~\ref{eq:safe_reward}. We use $\beta=4$ and a transfer function $\sigma$ such that $\sigma(\eta) = 0$ if $\eta < 5000$ (cf. Fig~\ref{fig:vaseworld_dpot_nonoise}), and $1$ otherwise. 

The policy is parameterized by a 3-layer deep 256-unit wide (fully connected) ReLU network and trained via Duelling Double Deep Q-Learning\footnote{We adapt the implementation due to \citet{deeprl}.} \citep{van2016deep, wang2015dueling}. The discount factor is set to $0.99$ and the target network is updated once every 200 iterations. For exploration, we use a $1 - \epsilon$ greedy policy, where $\epsilon$ is decayed linearly from $1$ to $0.1$ in the span of the first $10000$ iterations. The replay buffer stores $10000$ experiences and the batch-size used is $10$. Fig~\ref{fig:vaseworld_rl_goal} shows the probability of reaching the goal (in an episode of 30 steps) over the iterations (sample size $100$), whereas Fig~\ref{fig:vaseworld_rl_vase} shows the expected number of vases broken per episode (over the same 100 episodes). Both curves are smoothed by a Savitzky-Golay filter \citep{savitzky1964smoothing} of order $3$ and window-size $53$ (the original, unsmoothed curves are shaded). As expected, we find that using the safety penalty does indeed result in fewer vases broken, but also makes the task of reaching the goal difficult (we do not ensure that the goal is reachable without breaking vases).  
\begin{wrapfigure}{r}{0.48\textwidth}
\centering
\includegraphics[width=0.48\textwidth]{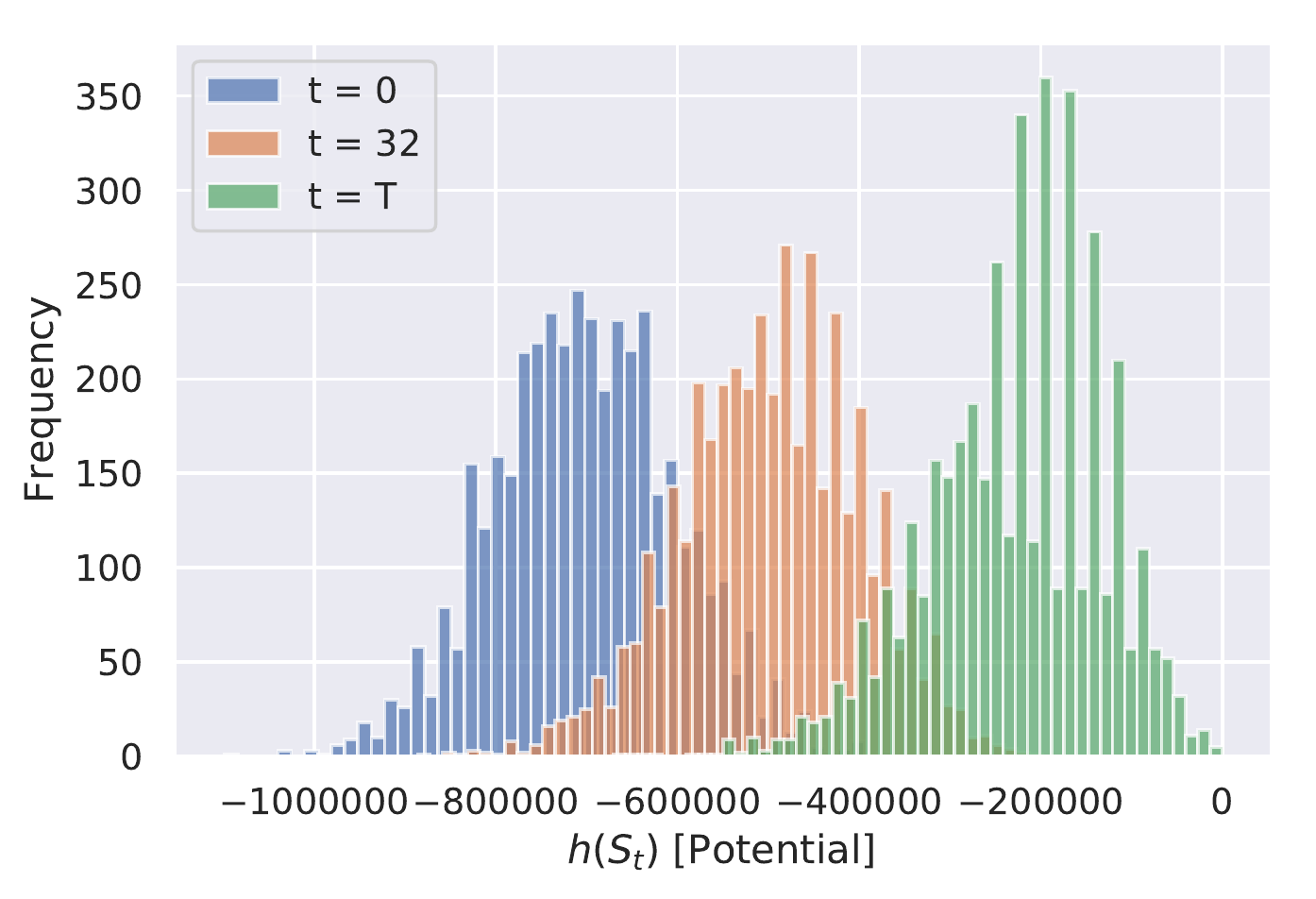}
\caption{\small Histogram (over trajectories) of values taken by $h$ at time-steps $t = 0$, $t = 32$ and $t = T = 128$.}\label{fig:vaseworld_pothist_nonoise}
\vspace{-20pt}
\end{wrapfigure}

\begin{figure}
%\vspace{-15pt}
\centering
\begin{subfigure}{0.5\textwidth}
\centering
  \includegraphics[width=1.\textwidth]{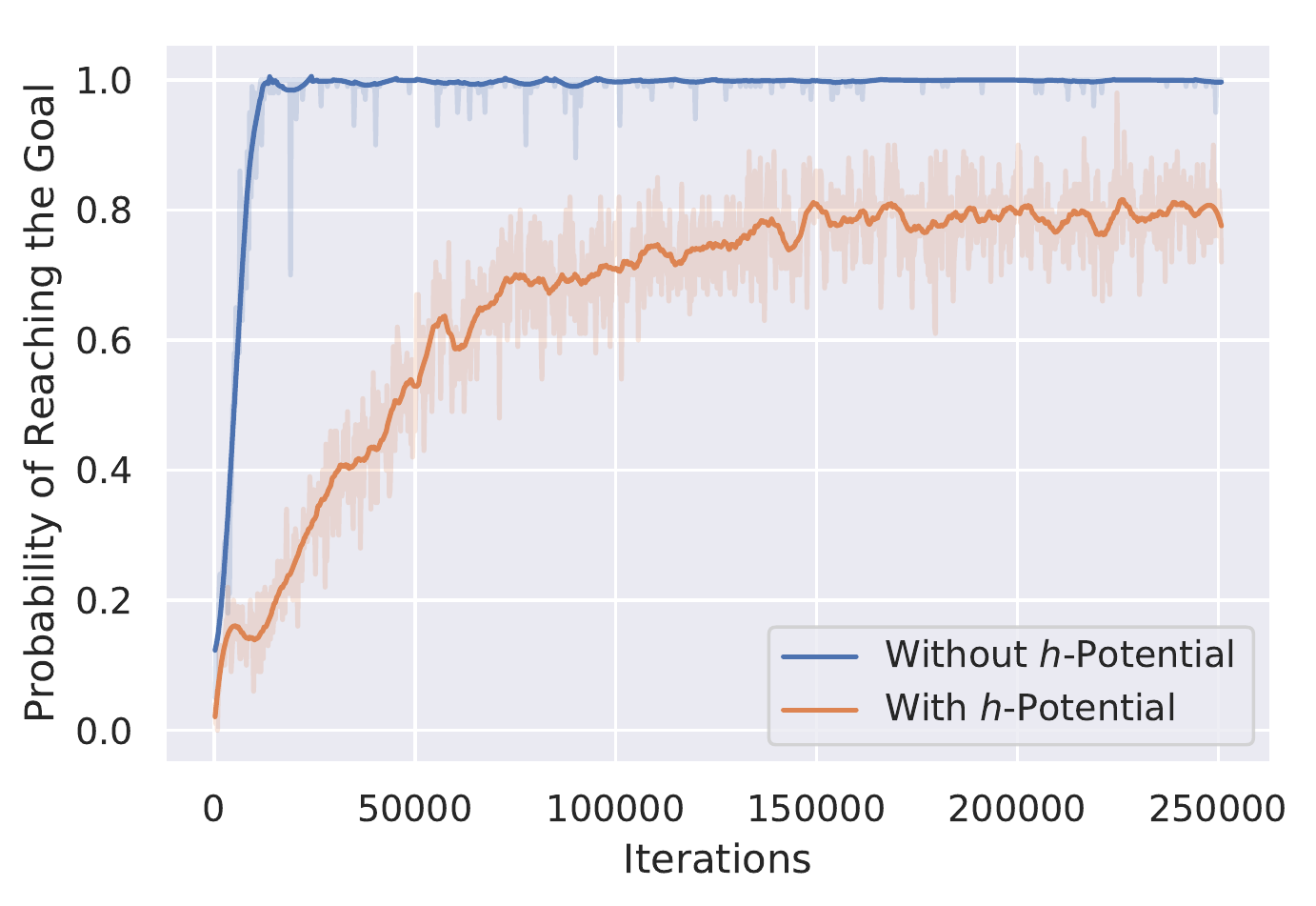}
\caption{\small Probability of reaching the goal. \label{fig:vaseworld_rl_goal}}
\end{subfigure}\hfill
\begin{subfigure}{0.5\textwidth} 
\centering
\includegraphics[width=1.\textwidth]{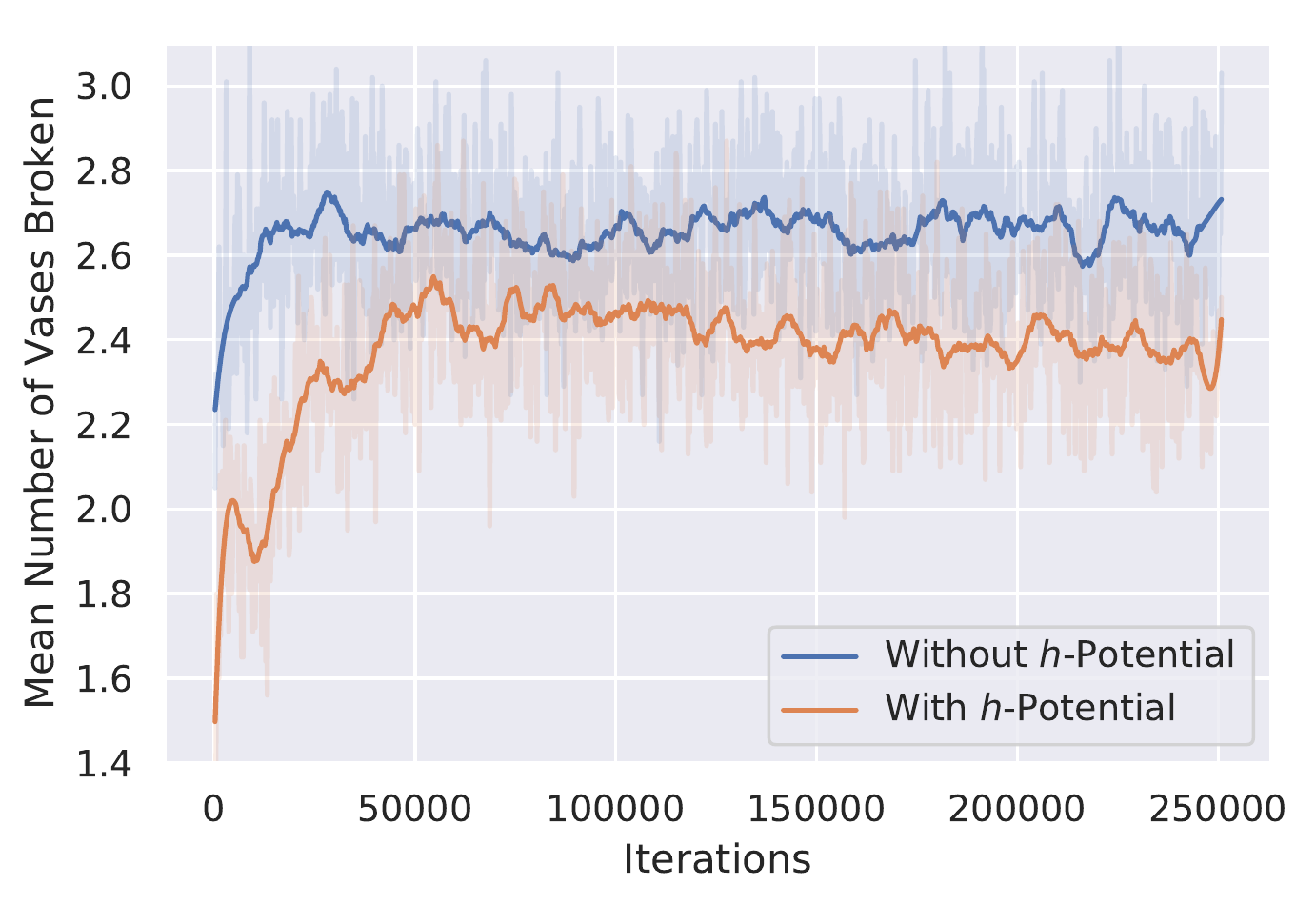}
\caption{\small Number of vases broken. \label{fig:vaseworld_rl_vase}}
\end{subfigure}
\caption{\small Probability of reaching the goal and the expected number of vases broken, obtained over $100$ evaluation episodes (per step). \textbf{Gist:} while the safety Lagrangian results in fewer vases broken, the probability of reaching the goal state is compromised. This trade-off between safety and efficiency is expected (cf. \citet{moldovan2012safe}). \label{fig:vaseworld_rl}}
\vspace{-15pt}
\end{figure}

\subsubsection{2D World with Drying Tomatoes} \label{app:tomatoworld}

The environment considered comprises a $7 \times 7$ 2D world where each cell is initially occupied by watered tomato plant\footnote{We draw inspiration from the tomato-watering environment described in \citet{leike2017ai}.}. The agent waters the cell it occupies, restoring the moisture level of the plant in the said cell to $100\%$. However, for each step the agent does not water a plant, it loses some moisture (by $2\%$ of maximum in our experiments). If a plant loses all moisture, it is considered dead and no amount of watering can resurrect it. The state-space comprises two $7 \times 7$ images: the first image is an indicator of the agent's position, whereas the pixel values of the second image quantifies the amount of moisture held by the plant\footnote{This is a strong causal signal which may distract the model. We include it nonetheless to make the task more challenging.} at the corresponding location.

We show that it is possible to recover an intrinsic reward signal that coincides well with one that one might engineer. To that end, we parameterize the $h$-potential as a two-layer deep 256-unit wide ReLU network and train it on 4096 trajectories (generated by a random policy) of length 128 for 10000 iterations of Adam (learning rate: 0.0001). The batch-size is set to 128 and the model is regularized with the trajectory regularizer ($\lambda = 0.5$). 

Unsurprisingly, we find that $h$ increases as the plants lose moisture. But conversely, when the agent waters a plant, it causes the $h$-potential to decrease by an amount that strongly correlates with the amount of moisture the watered plant gains. This can be used to define a \emph{dense} reward signal for the agent: 
\begin{wrapfigure}{r}{0.5\textwidth}
\centering
\includegraphics[width=0.5\textwidth]{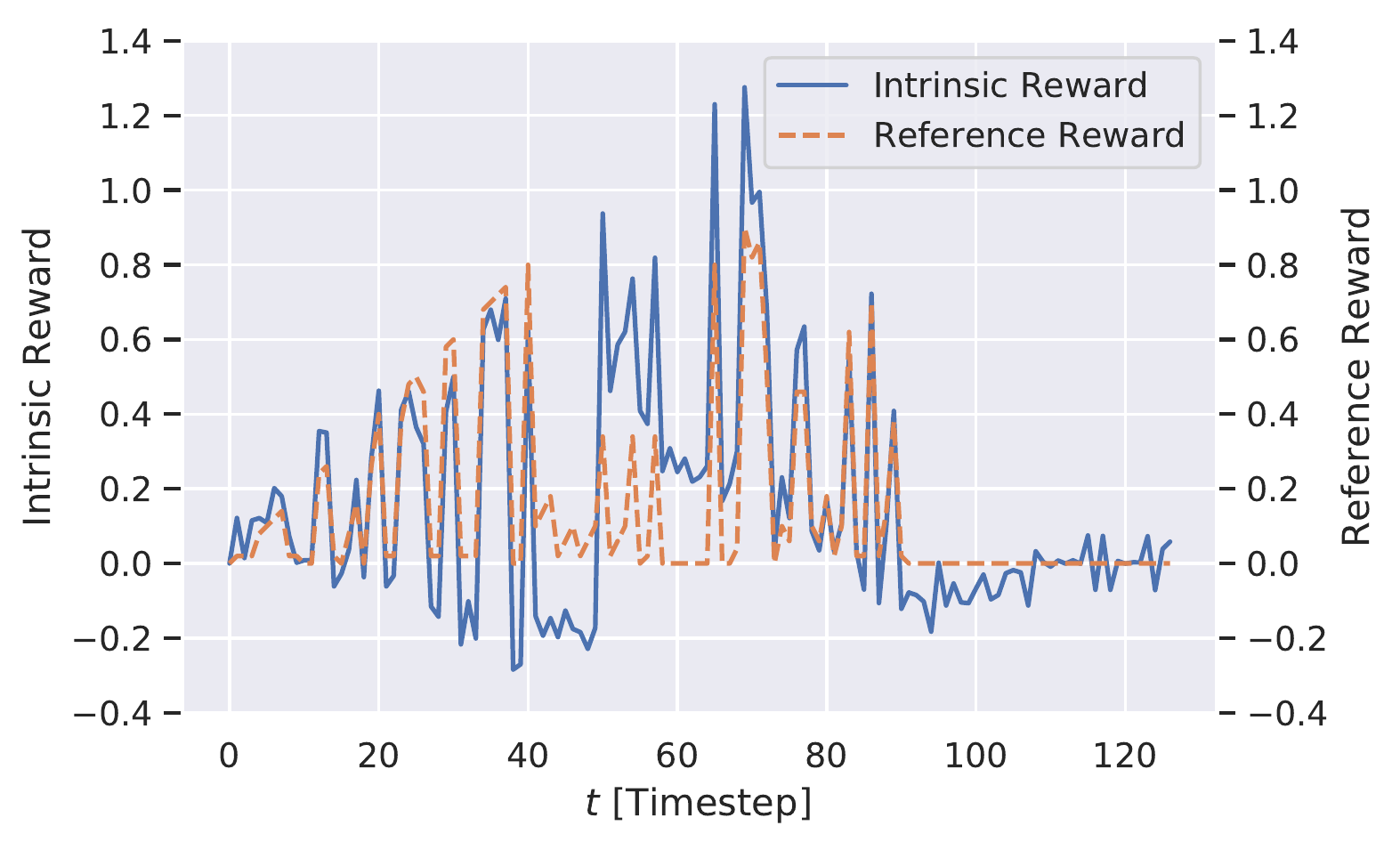}
\caption{\small The intrinsic reward (Eqn~\ref{eq:tomato_reward}) plotted against an engineered reward, which in this case is the amount of moisture gained by the tomato plant the agent just watered. \textbf{Gist:} the $h$-Potential captures useful information about the environment, which can then be utilized to define intrinsic rewards.}\label{fig:tomatoworld_intr_reward}
\vspace{-28pt}
\end{wrapfigure}
\beq \label{eq:tomato_reward}
\hat r_t = - \{\eta(s_{t - 1} \to s_{t}) - \text{RunningAverage}_{t}[\eta]\}
\eeq
where we use a momentum of $0.95$ to evaluate the running average. 

In Fig~\ref{fig:tomatoworld_intr_reward}, we plot for a random trajectory the intrinsic reward $\hat r_{t}$ against a reference reward, which in this case is the moisture gain of the plant the agent just watered. Further, we observe the reward function dropping significantly at around the 90-th iteration - this is precisely when all plants have died. This demonstrates that the $h$-potential can indeed be useful for defining intrinsic rewards. 

\subsubsection{Sokoban} \label{app:sokoban}

The environment state comprises five $10 \times 10$ binary images, where the pixel value at each location indicates the presence of the agent, a box, a goal, a wall and empty space. The layout of all sprites are randomized at each environment reset, under the constraint that the game is still solvable \citep{schrader2018sokoban}. The $h$-potential is parameterized by a two-layer deep and 512-unit wide network, which is trained on 4096 trajectories of length 512 for 20000 steps of Adam (learning rate: 0.0001). The batch-size is set to 256 and we use the trajectory regularizer ($\lambda = 0.05$) to regularize our model.

\begin{figure}
%\vspace{-15pt}
\centering
\begin{subfigure}{0.48\textwidth}
\centering
  \includegraphics[width=1.\textwidth]{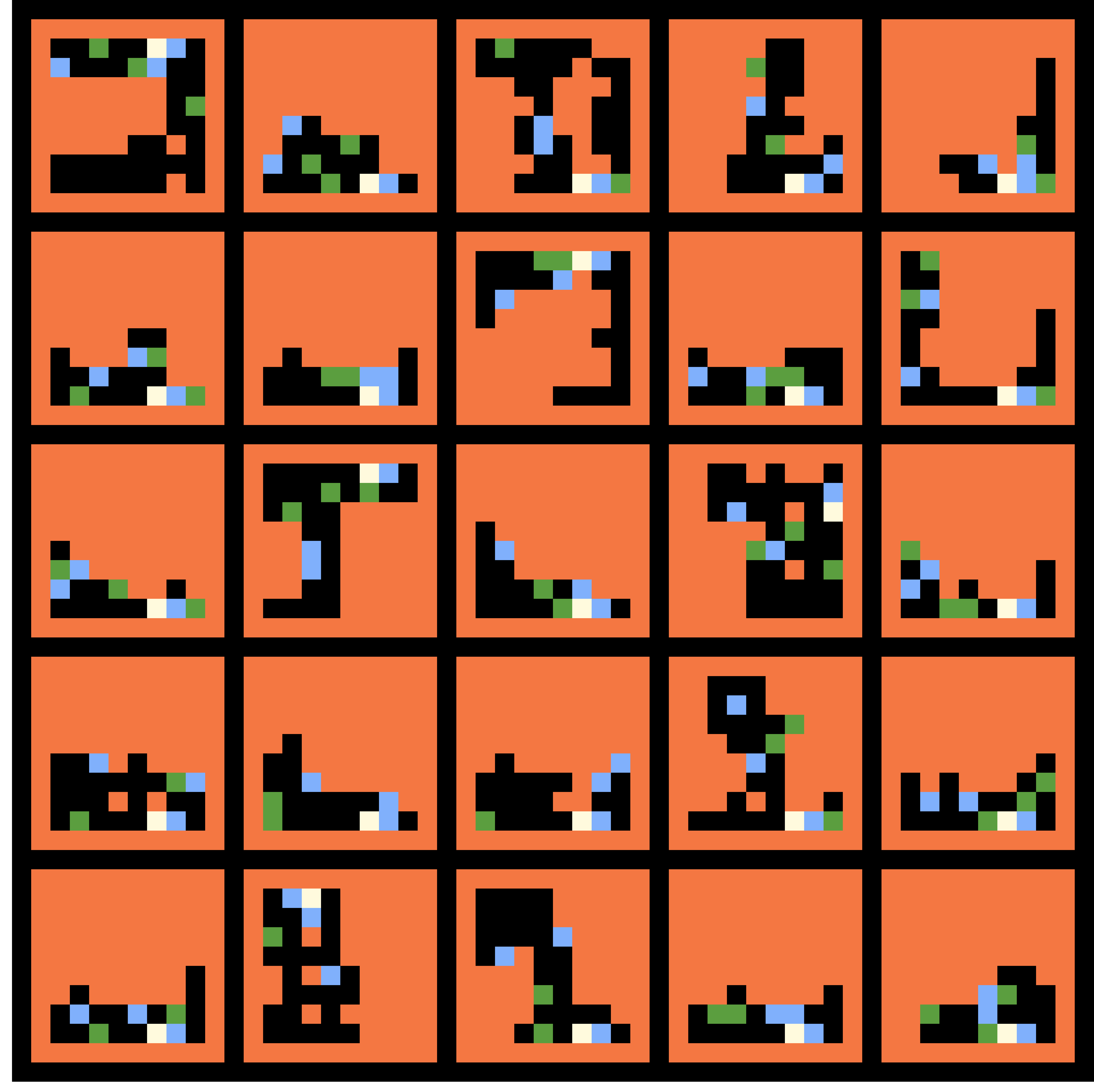}
\caption{\small States at time $t$. \label{fig:sokoban_trans_t}}
\end{subfigure}\hfill
\begin{subfigure}{0.48\textwidth} 
\centering
\includegraphics[width=1.\textwidth]{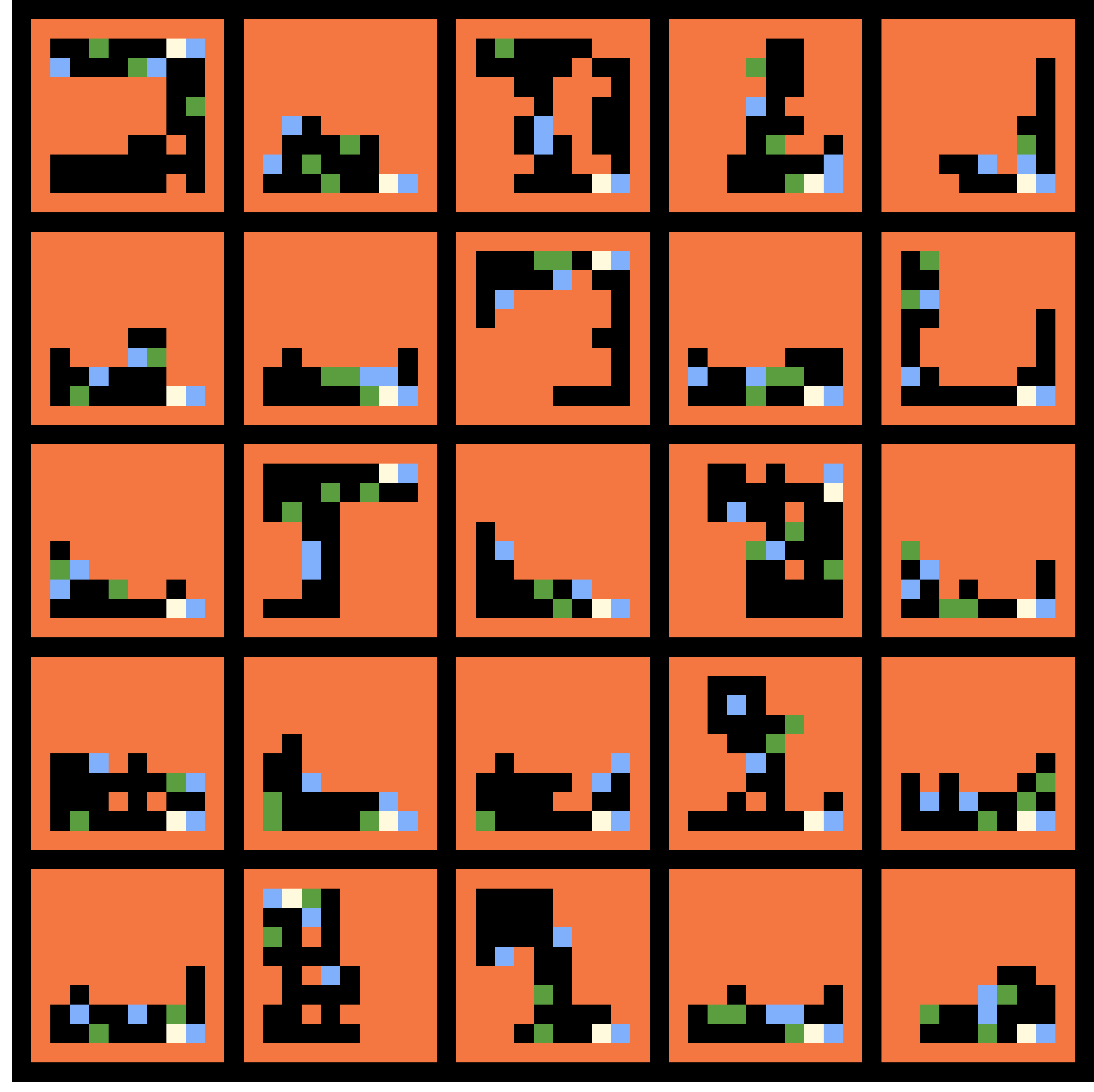}
\caption{\small States at time $t + 1$. \label{fig:sokoban_trans_tp1}}
\end{subfigure}
\caption{\small Random samples from $200$ transitions that cause the largest increase in the $h$-potential (out of a sample size of $8000$ transitions). The orange, white, blue and green sprites correspond to a wall, the agent, a box and a goal marker respectively. \textbf{Gist:} pushing boxes against the wall increases the $h$-potential. \label{fig:sokoban_trans}}
\vspace{-15pt}
\end{figure}

\subsection{Continous Environments}
\subsubsection{Under-damped Pendulum} \label{app:damped_pendulum}

\begin{figure}
%\vspace{-15pt}
\centering
\begin{subfigure}{0.48\textwidth}
\centering
  \includegraphics[width=1.\textwidth]{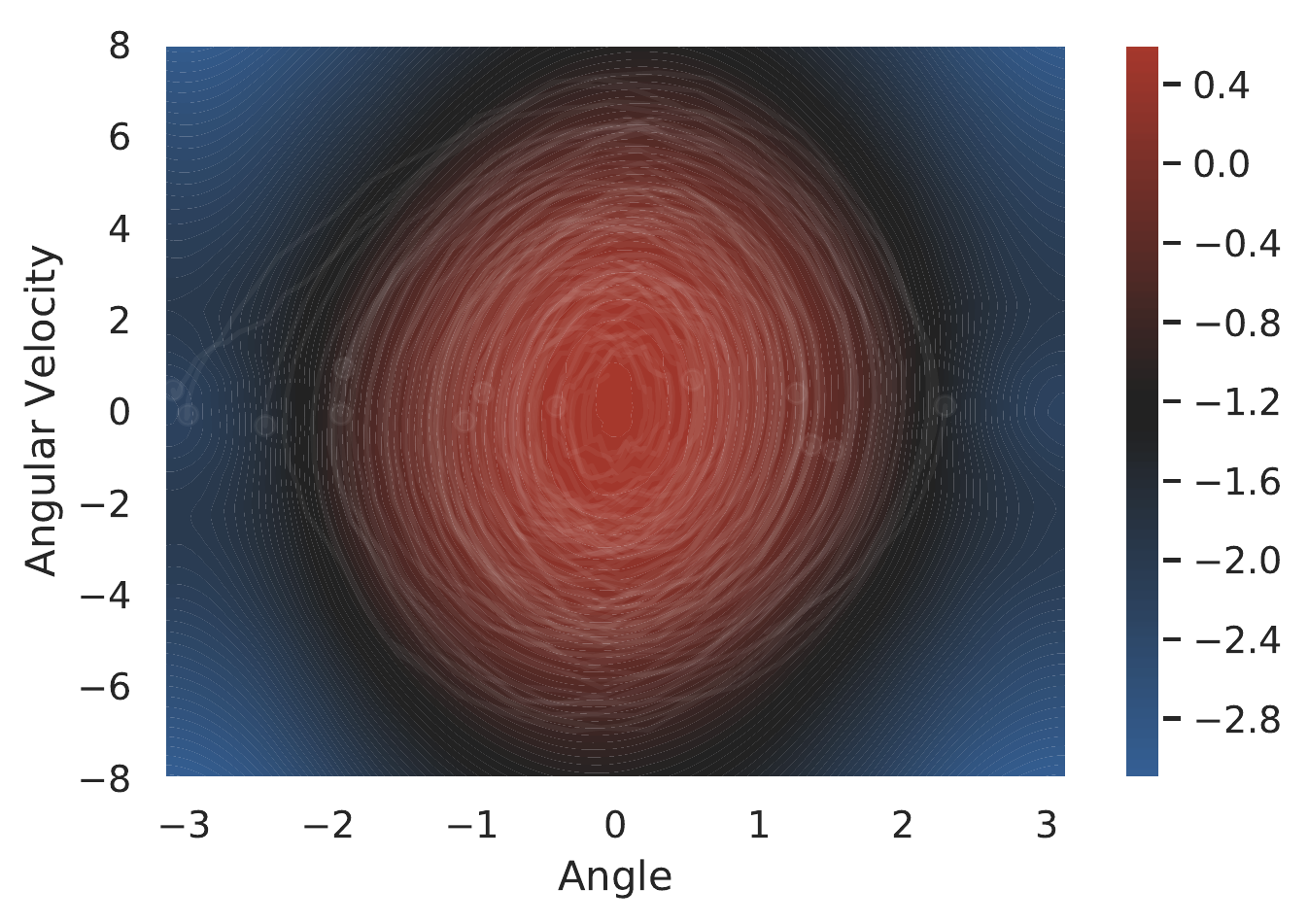}
\caption{\small Learned $h$-Potential as a function of the state-space $(\theta, \dot \theta)$. Overlaid are trajectories from a random policy. \label{fig:pend_pot_trace}}
\end{subfigure}\hfill
\begin{subfigure}{0.48\textwidth} 
\centering
\includegraphics[width=1.\textwidth]{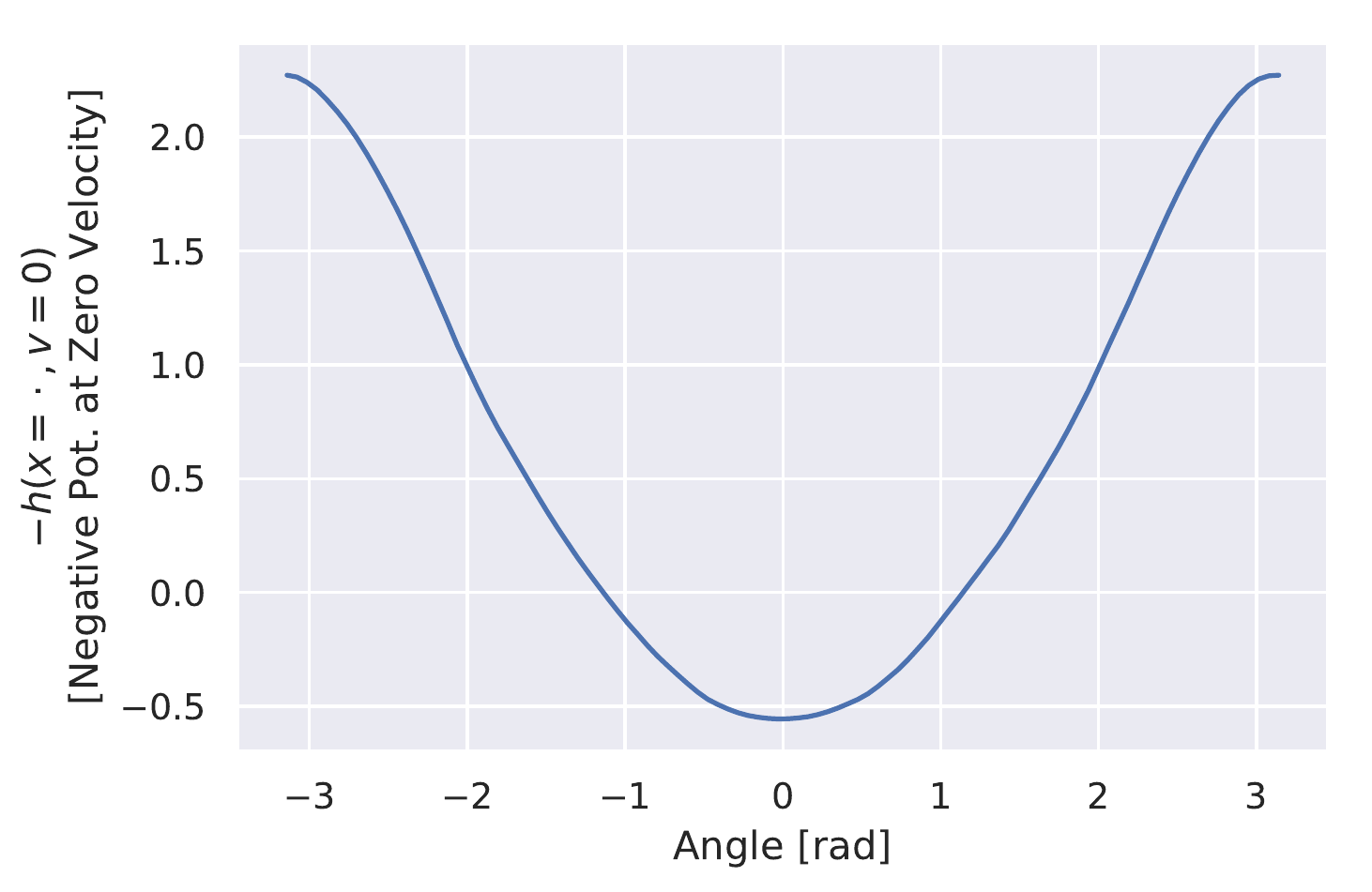}
\caption{\small Negative of the learned $h$-Potential as a function of $\theta$ when $\dot \theta = 0$. \label{fig:pend_pot_at_zvel}}
\end{subfigure}
\caption{\small \textbf{Gist:} the learned $h$-Potential takes large values around $(\theta, \dot \theta) = \mathbf{0}$, since that is where most trajectories terminate due to the effect of damping.  \label{fig:pend_pot}}
\end{figure}

\textbf{Under-damped Pendulum.} The environment considered simulates an under-damped pendulum, where the state space comprises the angle\footnote{$\theta$ is commonly represented as $(\cos(\theta), \sin(\theta))$ instead of a scalar.} $\theta$ and angular velocity $\dot \theta$ of the pendulum. The dynamics are governed by the following differential equation where $\tau$ is the (time-dependent) torque applied by the agent and $m$, $l$, $g$ are constants: 
\beq
\Ddot{\theta} = \frac{-3g}{2l} \sin(\theta) + \frac{3\tau}{ml^2} - \alpha \dot \theta
\eeq
We adapt the implementation in OpenAI Gym \citep{brockman2016openai} to add an extra term $\alpha \dot \theta$ to the dynamics to simulate friction. In our experiments, we set $g = 10$, $m = l = 1$, $\alpha = 0.1$ and the torque $\tau$ is uniformly sampled iid. from the interval $[-2, 2]$. 
\begin{wrapfigure}{r}{0.5\textwidth}
\centering
\includegraphics[width=0.5\textwidth]{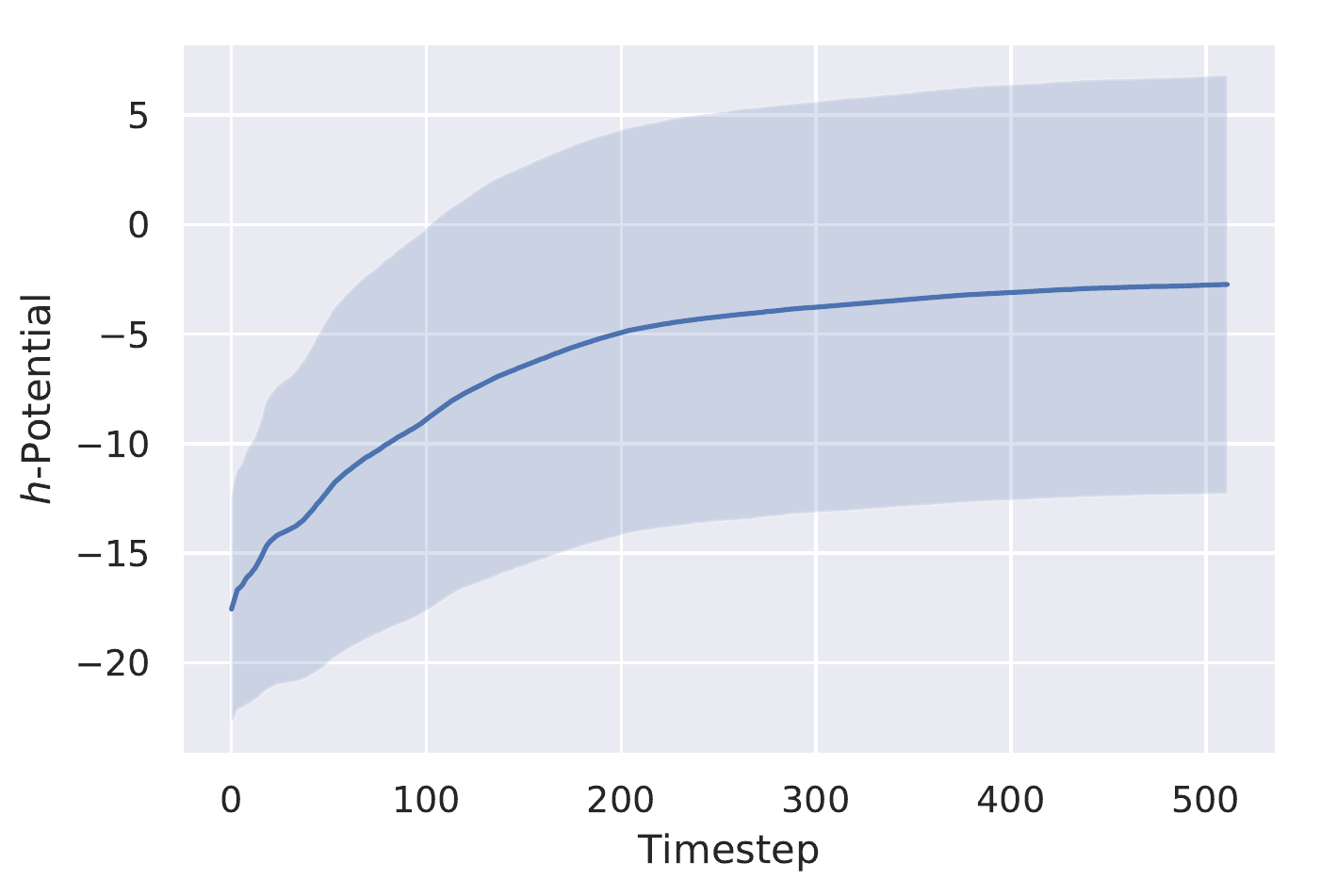}
\caption{\small $h$-Potential averaged over $8000$ trajectories, plotted against timestep $t$; shaded band shows the standard deviation. \textbf{Gist:} as required by its objective (Eqn~\ref{eq:objective_wo_reg}), the $h$-Potential must increase in expectation along trajectories.}\label{fig:sokoban_pot}
\vspace{-40pt}
\end{wrapfigure}

% \begin{wrapfigure}{r}{0.5\textwidth}
% \centering
% \includegraphics[width=0.5\textwidth]{figures/pendulum/pend_pot_trace_filled.pdf}
% \caption{\small Lorem Ipsum.}\label{fig:pend_pot_trace}
% \end{wrapfigure}

The $h$-Potential is parameterized by a two-layer 256-unit wide ReLU network, which is trained on 4096 trajectories of length 256 for 20000 steps of stochastic gradient descent with Adam (learning rate: 0.0001). The batch-size is set to 1024 and we use the trajectory regularizer with $\lambda = 1$ to regularize the network. Fig~\ref{fig:pend_pot_trace} plots the learned $h$-potential (trained with trajectory regularizer) as a function of the state $(\theta, \dot \theta)$ whereas Fig~\ref{fig:pend_pot_at_zvel} shows the negative potential for all angles $\theta$ at zero angular velocity, i.e. $\dot \theta = 0$. We indeed find that states in the vicinity of $\theta = 0$ have a larger $h$-potential, owing to the fact that all trajectories converge to $(\theta, \dot \theta) = \mathbf{0}$ for large $t$ due to the dissipative action of friction. 

% \begin{wrapfigure}{r}{0.5\textwidth}
% \centering
% \includegraphics[width=0.5\textwidth]{figures/pendulum/pend_pot_at_zvel.pdf}
% \caption{\small $h$-Potential plotted as a function of $\theta$ with fixed $\dot \theta = 0$.}\label{fig:pend_pot_at_zvel}
% \end{wrapfigure}

\subsubsection{Continuous Mountain Car} \label{app:mountain_car}
The environment\footnote{We adapt the implemetation due to \citet{brockman2016openai}, available here: \texttt{github.com/openai/gym/blob/master/gym/envs/classic\_control/continuous\_mountain\_car.py}} considered is a variation of Mountain Car \citep{sutton2011reinforcement}, where the state-space is a tuple $(x, \dot x)$ of the position and velocity of a vehicle on a mountainous terrain. The action space is the interval $[-1, 1]$ and denotes the \emph{force} $f$ applied by the vehicle. The dynamics of the modified environment is given by the following equation of motion: 
\beq
\Ddot{x} = \zeta f - 0.0025 \cos{3x} - \alpha \dot x
\eeq
where $\zeta$ and $\alpha$ are constants set to $0.0015$ and $0.1$ respectively, and the velocity $\dot x$ is clamped to the interval $[-0.07, 0.07]$. Our modification is the last $\alpha \dot x$ term to simulate friction. Further, the initial state $(x, \dot x)$ is sampled uniformly from the state space $\cS = [-1.2, 0.6] \times [-0.07, 0.07]$. This can potentially be avoided if an exploratory policy is used (instead of the random policy) to gather trajectories, but we leave this for future work. 

The $h$-potential is parameterized by a two-layer 256-unit wide ReLU network, which is trained on 4096 trajectories of length 256 for 20000 steps of stochastic gradient descent with Adam (learning rate: 0.0001). The batch-size is set to 1024 and we use the trajectory regularizer with $\lambda = 1$. 

\subsection{Stochastic Processes} \label{app:jko}
The environment state comprises two scalars, the $x_1$ and $x_2$ coordinates of the particle's position $\vx$. The potential is given by: 
\beq \label{eq:rw_pot}
\Psi(\vx) = \frac{x_1^2}{20} + \frac{x_2^2}{40}
\eeq
corresponding to a two dimensional Ornstein-Uhlenbeck process, and $\sqrt{2\beta^{-1}}$ is set to $0.3$.
\begin{wrapfigure}{r}{0.5\textwidth}
\centering
\includegraphics[width=0.5\textwidth]{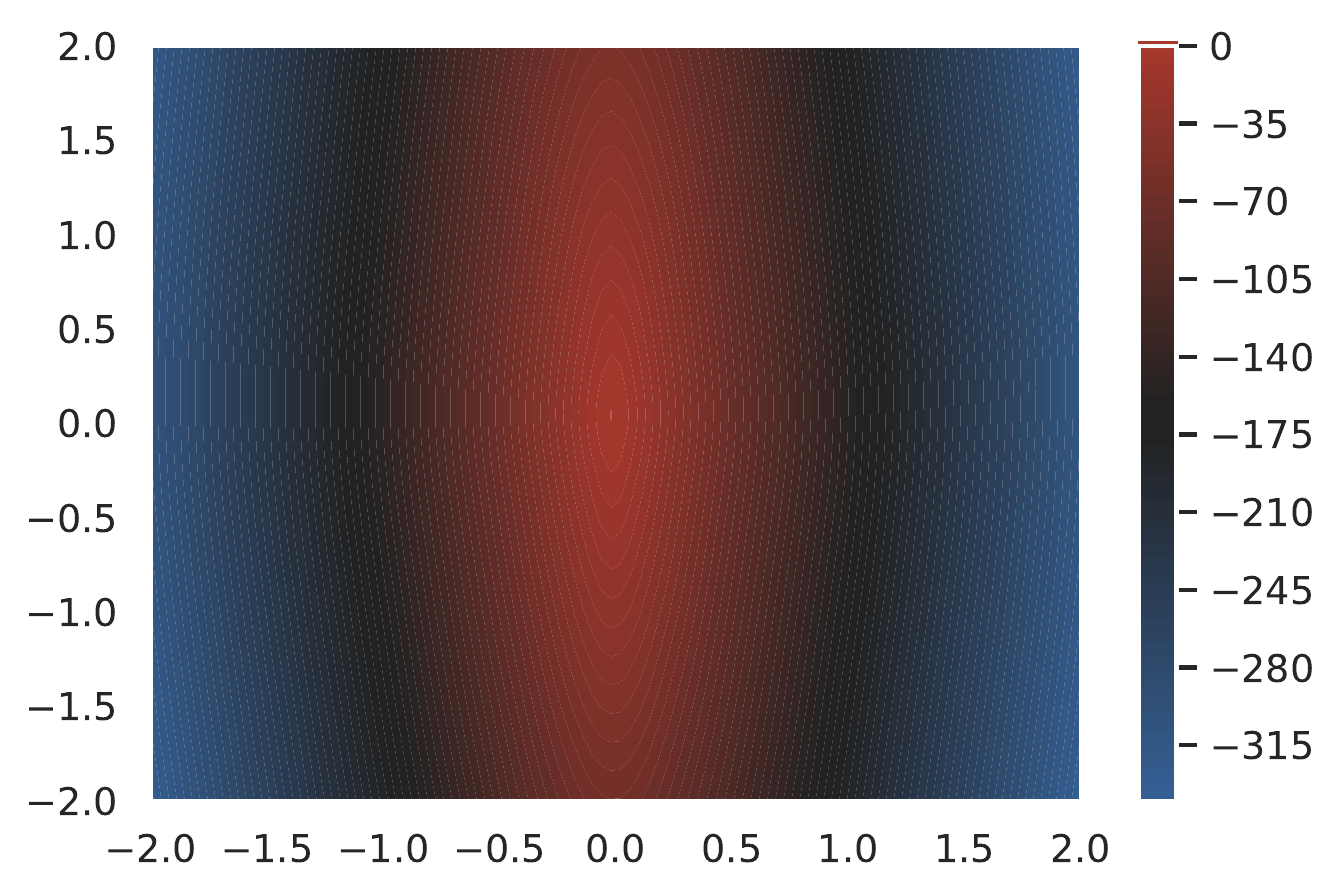}
\caption{\small Learned $h$-Potential as a function of position $\vx$. Observe the qualitative similarity to the potential $\Psi$ defined in Eqn~\ref{eq:rw_pot}.}\label{fig:rw_pot_filled}
\end{wrapfigure}

We train a two-layer deep, 512-unit wide network on 8092 trajectories of length 64 for 20000 steps of stochastic gradient descent with Adam (learning rate: 0.0001). The batch-size is set to 1024 and the network is regularized by weight decay (with coefficient $0.0005$). Fig~\ref{fig:rw_pot_filled} shows the learned $h$-potential as a function of position $\vx$. Fig~\ref{fig:rw_lyapunov} compares the free-energy functional with the learnt arrow of time given by the linearly scaled $H$-functional. To obtain the linear scaling parameters for the $H$, we find parameters $w$ and $b$ such that $\sum_{t = 0}^{N} (w H[\rho(\cdot, t)] + b - F[\rho(\cdot, t)])^2$ is minimized (constraining $w$ to be positive). 
% We now elaborate why it is justified to arbitrarily shift and scale $H$. First, observe that the Ito SDE (Eqn~\ref{eq:ito}) is invariant to constant shifts in $\Psi$: in other words, replacing $\Psi$ with $\Psi + \Phi$ (where $\Phi$ does not depend on $\vx$) leaves the dynamics unchanged. Consequently, $F[\rho] + \Phi$ is also a valid Lyapunov function (arrow of time), which makes it fair to arbitrarily shift the $H$-functional. Further, the scale of $h$ is in general determined by the regularization weight $\lambda$, 

\end{appendix}
\end{document}